\documentclass{article}
\usepackage[utf8]{inputenc}
\pdfoutput=1
\usepackage[sc,osf]{mathpazo}
\usepackage[height=8.5in,width=6in,letterpaper]{geometry}
\usepackage[sort&compress,numbers]{natbib}
\usepackage[colorlinks=true,citecolor=blue,breaklinks]{hyperref}
\usepackage[hyphenbreaks]{breakurl} %for arXiv's idiotic hyperref!
\usepackage[tracking]{microtype}

\makeatletter
\newlength\aftertitskip     \newlength\beforetitskip
\newlength\interauthorskip  \newlength\aftermaketitskip

\def\maketitle{\par
 \begingroup
   \def\thefootnote{\fnsymbol{footnote}}
   \def\@makefnmark{\hbox to 4pt{$^{\@thefnmark}$\hss}}
   \@maketitle \@thanks
 \endgroup
\setcounter{footnote}{0}
 \let\maketitle\relax \let\@maketitle\relax
 \gdef\@thanks{}\gdef\@author{}\gdef\@title{}\let\thanks\relax}

\def\@startauthor{\noindent \normalsize\bf}
\def\@endauthor{}
\def\@starteditor{\noindent \small {\bf Editor:~}}
\def\@endeditor{\normalsize}
\def\@maketitle{\vbox{\hsize\textwidth
 \linewidth\hsize \vskip \beforetitskip
 {\begin{center} \LARGE\@title \par \end{center}} \vskip \aftertitskip
 {\def\and{\unskip\enspace{\rm and}\enspace}%
  \def\addr{\small\it}%
  \def\email{\hfill\small\tt}%
  \def\name{\normalsize\bf}%
  \def\AND{\@endauthor\rm\hss \vskip \interauthorskip \@startauthor}
  \@startauthor \@author \@endauthor}
}}

\makeatother

\pdfoutput=1                    % force pdflatex to run

\usepackage{floatrow} % For putting table and figure sude by side
\newfloatcommand{capbtabbox}{table}[][\FBwidth]

\usepackage{graphicx}
\usepackage{amsmath,amsthm,amssymb,bm} % define this before the line numbering.
\usepackage{amsfonts}
\usepackage{mathrsfs}
\usepackage{xspace}
\usepackage{array}
\usepackage{enumerate}
\usepackage{algorithm}
\usepackage{algorithmic}
\usepackage{stmaryrd}
\usepackage{appendix}
\usepackage{wrapfig}
\usepackage{booktabs}
\usepackage{caption}
\usepackage{subcaption}
\numberwithin{equation}{section}

\newcommand{\calr}{\mathcal{R}} % only because \cr already taken
     \newcommand{\cf}{\mathcal{F}}         \newcommand{\co}{\mathcal{O}}         \newcommand{\cy}{\mathcal{Y}} 

\newcommand{\var}{\textrm{Var}}

\theoremstyle{plain}
\newtheorem{theorem}{Theorem}
\newtheorem{corollary}[theorem]{Corollary}

\newtheorem{lemma}[theorem]{Lemma}

\DeclareMathOperator{\tr}{tr}  % trace

\theoremstyle{definition}

\theoremstyle{remark}

\newcommand{\refsec}[1]{Sec.~\ref{#1}}
\newcommand{\reffig}[1]{Fig.~\ref{#1}}

\newcommand{\eps}{\varepsilon}

\newcommand{\nys}{Nystr\"om\xspace}

\newcommand{\dpp}{\textsc{Dpp}\xspace}

\newcommand{\rin}{r^{\text{in}}}
\newcommand{\rout}{r^{\text{out}}}
\newcommand{\tin}{t^{\text{in}}}
\newcommand{\tout}{t^{\text{out}}}
\newcommand{\nlsum}{\sum\nolimits}

\newcommand{\unif}{\texttt{Unif}\xspace}
\newcommand{\adap}{\texttt{AdapFull}\xspace}
\newcommand{\adappart}{\texttt{AdapPart}\xspace}
\newcommand{\kmeans}{\texttt{Kmeans}\xspace}
\newcommand{\lev}{\texttt{Lev}\xspace}
\newcommand{\rlev}{\texttt{RegLev}\xspace}
\newcommand{\applev}{\texttt{AppLev}\xspace}
\newcommand{\apprlev}{\texttt{AppRegLev}\xspace}
\newcommand{\ours}{\texttt{kDPP}\xspace}

\title{Fast \dpp Sampling for \nys\\ with Application to Kernel Methods}
\author{\name Chengtao Li \email{ctli@mit.edu}\\
  \name Stefanie Jegelka \email{stefje@csail.mit.edu}\\
  \name Suvrit Sra \email{suvrit@mit.edu}\\
  \addr{Massachusetts Institute of Technology}
}

\begin{document}
\maketitle

\begin{abstract}
  The \nys method has long been popular for scaling up kernel methods. Its theoretical guarantees and empirical performance rely critically on the quality of the \emph{landmarks} selected. We study landmark selection for \nys using Determinantal Point Processes (\dpp{}s), discrete probability models that allow tractable generation of \emph{diverse} samples. We prove that landmarks selected via \dpp{}s guarantee bounds on approximation errors; subsequently, we analyze implications for kernel ridge regression. Contrary to prior reservations due to cubic complexity of \dpp sampling, we show that (under certain conditions) Markov chain \dpp sampling requires only \emph{linear} time in the size of the data. We present several empirical results that support our theoretical analysis, and demonstrate the superior performance of \dpp-based landmark selection compared with existing approaches.
\end{abstract}

\section{Introduction}
\label{sec:introduction}
Low-rank matrix approximation is an important ingredient of modern machine learning methods. Numerous learning tasks rely on multiplication and inversion of matrices, operations that scale cubically in the number of data points $N$, and therefore quickly become a bottleneck for large data. In such cases, low-rank matrix approximations promise speedups with a tolerable loss in accuracy.

A notable instance is the \emph{Nystr\"om method} \citep{nystrom1930praktische,williams2001using}, which takes a positive semidefinite matrix $K\in\mathbb{R}^{N\times N}$ as input, selects from it a small subset $C$ of columns $K_{\cdot,C}$, and constructs the approximation $\tilde{K} = K_{\cdot,C}K_{C,C}^\dagger K_{C,\cdot}$. The matrix $\tilde{K}$ is then used in place of $K$, which can decrease runtimes from $\co(N^3)$ to $\co(N|C|^3)$, a huge savings (since typically $|C|\ll N$).

Since its introduction into machine learning, the Nystr\"om method has been applied to a wide spectrum of  problems, including kernel ICA \cite{bach2003kernel,shen2009fast}, kernel and spectral methods in computer vision \cite{belabbas2009landmark,fowlkes2004spectral}, manifold learning~\cite{talwalkar2008large,talwalkar2013large},  regularization~\cite{rudi2015less}, and efficient approximate sampling \cite{affandi2013nystrom}. Recent work~\cite{cortes2010impact,bach2012sharp,alaoui2014fast} shows risk bounds for Nystr\"om applied to various kernel methods.

The most important step of the Nystr\"om method is the selection of the subset $C$, the so-called \emph{landmarks}. This choice governs the approximation error and subsequent performance of the approximated learning methods~\cite{cortes2010impact}. The most basic strategy is to sample landmarks uniformly at random~\cite{williams2001using}. More sophisticated non-uniform selection strategies include
deterministic greedy schemes \cite{smola2000sparse}, incomplete Cholesky decomposition \cite{fine2002efficient,bach2005predictive}, sampling with probabilities proportional to diagonal values~\cite{drineas2005nystrom} or to column norms~\cite{drineas2006fast}, sampling based on leverage scores~\cite{gittens2013revisiting}, via K-means~\cite{zhang2008improved}, or using submatrix determinants~\cite{belabbas2009spectral}.

We study landmark selection using \emph{Determinantal Point Processes (\dpp)}, discrete probability models that allow tractable sampling of diverse non-independent subsets~\cite{macchi1975coincidence,kulesza2012determinantal}. Our work generalizes the determinant based scheme of~\citet{belabbas2009spectral}.\footnote{The authors do not make any connection to \dpp{}s.} We refer to our scheme as \dpp-\nys, and analyze it from several perspectives.

A key quantity in our analysis is the error of the \nys approximation. Suppose $k$ is the target rank; then for selecting $c\ge k$ landmarks, \nys's error is typically measured using the Frobenius or spectral norm relative to the best achievable error via rank-$k$ SVD 
$K_k$; i.e., we measure 
\begin{align*}\label{eq:frospebound}
  {\|K - K_{\cdot,C}K_{C,C}^\dagger K_{C,\cdot}\|_F\over \|K - K_k\|_F} \quad \text{ or } \quad
  {\|K - K_{\cdot,C}K_{C,C}^\dagger K_{C,\cdot}\|_2\over \|K - K_k\|_2}.
\end{align*}
Several authors also use additive instead of relative bounds. However, such bounds are very sensitive to scaling, and become loose even if a single entry of the matrix is large. Thus, we focus on the above relative error bounds.

First, we analyze this approximation error. Previous analyses~\cite{belabbas2009spectral} fix a cardinality $c=k$; we allow the general case of selecting $c \ge k$ columns. Our relative error bounds rely on the properties of characteristic polynomials. Empirically, \dpp-\nys obtains approximations competitive to state-of-the-art methods.

Second, we consider its impact on kernel methods. 
Specifically, we address the impact of \nys-based kernel approximations on kernel ridge regression. This task has been noted as the main application in~\cite{bach2012sharp,alaoui2014fast}. 
We show risk bounds of \dpp-\nys that hold in expectation. Empirically, it achieves the best performance among competing methods.

Third, we consider the efficiency of \dpp-\nys; specifically, its tradeoff between error and running time. Since its proposal, determinantal sampling has so far not been used widely in practice due to valid concerns about its scalability. We consider a Gibbs sampler for $k$-\dpp, and analyze its mixing time using a \emph{path coupling}~\cite{bubley1997path} argument. We prove that under certain conditions the chain is fast mixing, which implies a \emph{linear} running time for \dpp sampling of landmarks. Empirical results indicate that the chain yields favorable results within a small number of iterations, and the best efficiency-accuracy traedoffs compared to state-of-art methods (Figure~\ref{fig:tradeoff}). 

\vspace*{-5pt}
\section{Background and Notation}
Throughout, we are approximating a given positive semidifinite (PSD) matrix  $K\in\mathbb{R}^{N\times N}$ with eigendecomposition $K = U\Lambda U^\top$ and eigenvalues $\lambda_1 \geq \ldots \geq \lambda_N$. We use $K_{i,\cdot}$ for the $i$-th row and $K_{\cdot,j}$ for the $j$-th column, and, likewise,
$K_{C,\cdot}$ for the rows of $K$ and $K_{\cdot,C}$ for the columns of $K$ indexed by $C\subseteq [N]$. Finally, $K_{C,C}$ is the submatrix of $K$ with rows and columns indexed by $C$. 
In this notation, $K_k = U_{\cdot,[k]}\Lambda_{[k],[k]}U_{\cdot,[k]}^\top$ is the best rank-$k$ approximation to $K$ in both Frobenius and spectral norm. 
We write $r(\cdot)$ for the rank and $(\cdot)^\dagger$ for the pseudoinverse, and denote a decomposition of $K$ by $B^\top B$, where $B\in\mathbb{R}^{r(K)\times N}$. 

\textbf{The Nystr\"om Method.}
The \emph{standard Nystr\"om} method selects a subset $C\subseteq [N]$ of $c=|C|$ \emph{landmarks}, and approximates $K$ with $K_{\cdot,C} K_{C,C}^\dagger K_{C,\cdot}$. The actual set of landmarks affects the approximation quality, and is hence the subject of a substantial body of research \cite{cortes2010impact,smola2000sparse,fine2002efficient,bach2005predictive,drineas2005nystrom,drineas2006fast,gittens2013revisiting,zhang2008improved,belabbas2009spectral}.
Besides various landmark selection methods, there exist variations of the standard Nystr\"om method. The \emph{ensemble Nystr\"om method} \cite{kumar2009ensemble}, for instance, uses a weighted combination of approximations. The \emph{modified Nystr\"om method} constructs an approximation $K_{\cdot,C} K_{\cdot,C}^\dagger K K_{C,\cdot}^\dagger K_{C,\cdot}$ \cite{sun2015review}.
In this paper, we focus on the standard Nystr\"om method. 

\textbf{Determinantal Point Processes.}
A \emph{determinantal point process} $\dpp(K)$ is a distribution over all subsets of a ground set $\cy$ of cardinality $N$ that is determined by a PSD kernel $K\in\mathbb{R}^{N\times N}$. 
The probability of observing a subset $C\subseteq [N]$ is proportional to $\det(K_{C,C})$, that is,
\begin{align}
\Pr(C) = \det(K_{C,C})/\det(K+I).
\end{align}
When conditioning on a fixed cardinality, one obtains a $k$-\dpp~\cite{kulesza2011k}. To avoid confusion with the target rank $k$, and since we use cardinality $c=|C|$, we will refer to this distribution as $c$-\dpp\footnote{Note that we refer to \dpp-\nys as \ours in experimental parts.}, and note that
\begin{align*}
  \Pr(C \mid |C| = c) &= \det(K_{C,C})e_c(K)^{-1}\llbracket\, |C|=c\rrbracket,
\end{align*}
where $e_c(K)$ is the $c$-th coefficient of the characteristic polynomial $\det(\lambda I - K) = \sum_{j=0}^N(-1)^je_j(K)\lambda^{N-j}$. 

Sampling from a ($c$-)\dpp can be done in polynomial time, but requires a full eigendecomposition of $K$~\cite{hough2006determinantal}, which is prohibitive for large $N$.
A number of approaches have been proposed for more efficient sampling \cite{affandi2013nystrom,wang2014using,li2015efficient}. We follow an alternative approach based on Gibbs sampling and show that it can offer fast polynomial-time \dpp sampling and Nystr\"om approximations. 

\section{\dpp for the \nys Method}
\label{sec:dppnys}
Next, we consider sampling $c$ landmarks $C\subseteq [N]$ from $c$-\dpp{($K$)}, and use the approximation $\tilde{K} = K_{\cdot,C} K_{C,C}^\dagger K_{C,\cdot}$. We call this approach \dpp-\nys. It was essentially introduced in~\cite{belabbas2009spectral}, but without making the explicit connection to {\dpp}s. Our analysis builds on this connection and subsumes existing results that only apply to $c$ being the rank $k$ of the target approximation. 

We begin with error bounds for matrix approximations:
\begin{theorem}[Relative Error]\label{thm:nys}
If $C \sim c$-\dpp{($K$)}, then \dpp-\nys satisfies the relative error bounds
\begin{small}
\begin{align*}
\mathbb{E}_C\left[{\|K - K_{\cdot C} (K_{C,C})^\dagger K_{C\cdot}\|_F \over \|K - K_k\|_F}\right] &\le \left(\frac{c+1}{c+1-k}\right)\sqrt{N-k}, \\  
\mathbb{E}_C\left[{\|K - K_{\cdot C} (K_{C,C})^\dagger K_{C\cdot}\|_2 \over \|K - K_k\|_2}\right] &\le \left({c+1\over c+1-k}\right)(N-k).
\end{align*}
\end{small}
\end{theorem}

These bounds hold in expectation. An additional argument based on \cite{pemantle2014concentration} yields high probability bounds, too (Appendix~\ref{append:sec:proof}).

To show Theorem~\ref{thm:nys}, we exploit a property of characteristic polynomials observed in~\cite{guruswami2012optimal}. But first recall that the coefficients of characteristic polynomials satisfy 
$e_c(K) = \sum\nolimits_{|S| = c}\det(B_{\cdot,S}^\top B_{\cdot,S}) =  e_c(\Lambda)$.
\begin{lemma}[\protect{\citet{guruswami2012optimal}}]\label{lem:char}
  For any $c \geq k > 0$, it holds that
    \begin{align*}
    {e_{c+1}(K)\over e_c(K)} \le {1\over c + 1 - k}\sum_{i > k} \lambda_i.
    \end{align*}
\end{lemma}
With Lemma~\ref{lem:char} in hand, we are ready to prove Theorem~\ref{thm:nys}.
    
\begin{proof}[Proof (Thm.~\ref{thm:nys}).]
  We begin with the Frobenius norm error, and then show the spectral norm result. Using the decomposition $K = B^\top B$, it holds that
    \begin{align*}
    \mathbb{E}_C &\left[\|K - K_{\cdot C} K_{C,C}^\dagger K_{C\cdot}\|_F\right] = \mathbb{E}_C \left[\|B^\top B - B^\top B_{\cdot,C}(B_{\cdot,C}^\top B_{\cdot,C})^\dagger B_{\cdot,C}^\top B\|_F\right]\\
    &= \mathbb{E}_C \left[\|B^\top (I - B_{\cdot,C}(B_{\cdot,C}^\top B_{\cdot,C})^\dagger B_{\cdot,C}^\top) B\|_F\right]= \mathbb{E}_C \left[\|B^\top (I - U^C (U^C)^\top) B\|_F\right],
    \end{align*}
where $U^C\Sigma^C (V^C)^\top$ is the SVD of $B_{\cdot,C}$. Next, we extend $U^C\in\mathbb{R}^{r(K)\times c}$ to an orthogonal basis $[U^C\;(U^C)^\perp]\in\mathbb{R}^{r(K)\times r(K)}$ of $\mathbb{R}^N$. Using that  
$I - U^C (U^C)^\top = (U^C)^\perp ((U^C)^\perp)^\top$ and 
applying Cauchy-Schwartz yields
\begin{small}
  \begin{align*}
    \mathbb{E}_C &\left[\|B^\top (I - U^C (U^C)^\top) B\|_F \right]= \mathbb{E}_C \left[\|B^\top (U^C)^\perp ((U^C)^\perp)^\top B\|_F\right]\\
    &= \mathbb{E}_C \left[\sqrt{\nlsum_{i,j} (b_i^\top (U^C)^\perp ((U^C)^\perp)^\top b_j)^2}\right]\le \mathbb{E}_C \left[\sqrt{(\nlsum_{i,j} \|b_i^\top (U^C)^\perp\|_2^2 \|b_j^\top (U^C)^\perp\|_2^2)}\right]\\
    &= \mathbb{E}_C \left[\nlsum_i \|b_i^\top (U^C)^\perp\|_2^2\right]= {1\over e_c(K)}\nlsum_{|C| = c}\nlsum_{i} \det(B_{\cdot,C}^\top B_{\cdot,C}) \|b_i^\top (U^C)^\perp\|_2^2\\
    &\overset{(a)}= {1\over e_c(K)}\nlsum_{|C| = c}\nlsum_{i\notin C} \det(B_{\cdot,C\cup\{i\}}B_{\cdot,C\cup\{i\}}^\top)\\
    &\overset{(b)}{=} (c+1) {e_{c+1}(K)\over e_c(K)}.
  \end{align*}
\end{small}%
In $(a)$, we use that $(U^C)^\perp$ projects vectors onto the null (column) space of $B$, and $(b)$ uses the definition of $e_c$. With Lemma~\ref{lem:char}, it follows that
    \begin{align*}
    (c+1) &\tfrac{e_{c+1}(K)}{e_c(K)} \le \tfrac{c+1}{c+1-k}\nlsum_{i > k}\lambda_i\\
    &\le \tfrac{c+1}{c+1-k}\sqrt{N-k}\sqrt{\nlsum_{i > k} \lambda_i^2}= \tfrac{c+1}{c+1-k}\sqrt{N-k} \|K - K_k\|_F.
    \end{align*}

The bound on the Frobenius norm immediately implies the bound on the spectral norm:
\begin{small}
  \begin{align*}
    \mathbb{E}_C &\left[\|K - K_{\cdot C} (K_{C,C})^\dagger K_{C\cdot}\|_2\right] \;\;\le \mathbb{E}_C \left[\|K - K_{\cdot C} K_{C,C}^\dagger K_{C\cdot}\|_F \right]\\
    &\;\;\le {c+1\over c+1-k}\sqrt{N-k} \|K - K_k\|_F \;\;\le {c+1\over c+1-k}(N-k) \|K - K_k\|_2 \qedhere
  \end{align*}
\end{small}
\end{proof}

\paragraph{Remarks.} Compared to previous bounds (e.g.,~\cite{gittens2013revisiting} on uniform and leverage score sampling), our bounds seem somewhat weaker asymptotically (since as $c\to N$ they do not converge to 1). This suggests that there is an opportunity for further tightening our bounds, which may be worthwhile, given than in Section~\refsec{sec:exp:app} our extensive experiments on various datasets with \dpp-\nys  show that it attains superior accuracies compared with  various state-of-art methods.

\section{Low-rank Kernel Ridge Regression}
Our theoretical (Section \ref{sec:dppnys}) and empirical (Section~\ref{sec:exp:app}) results suggest that \dpp-\nys is well-suited for scaling kernel methods. In this section, we analyze its implications on kernel ridge regression. The experiments in Section~\ref{sec:exp} confirm our results empirically.

We have $N$ training samples $\{(x_i,y_i)\}_{i=1}^N$, where $y_i = z_i + \epsilon_i$ are the observed labels under zero-mean noise with finite covariance. We minimize a regularized empirical loss
\begin{align*}\vspace{-3pt}
  \min_{f\in\cf}{1\over N}\sum_{i=1}^N \ell(y_i,f(x_i)) + {\gamma\over 2}\|f\|^2
\end{align*}
over an RKHS $\mathcal{F}$. Equivalently, we solve the problem
\begin{align*}
\min_{\alpha\in\mathbb{R}^N}{1\over N} \sum_{i=1}^N \ell(y_i,(K\alpha)_i) + {\gamma\over 2}\alpha^\top K \alpha,
\end{align*}
for the corresponding kernel matrix $K$. With the squared loss
$\ell(y,f(x)) = {1\over 2}(y - f(x))^2$, the resulting estimator is
\begin{align}\label{eq:estimator}
\hat{f}(x) &= \sum_{i=1}^N \hat{\alpha}_i k(x,x_i),\quad \hat{\alpha} = (K + n\gamma I)^{-1}y,
\end{align}
and the prediction for $\{x_i\}_{i=1}^N$ is given by $\hat{z} = K(K + N\gamma I)^{-1}y\in\mathbb{R}^N$. Denoting the noise covariance by $F$, we obtain the risk
\begin{align}
  \nonumber
\mathcal{R}&(\hat{z}) = \tfrac{1}{N}\mathbb{E}_{\eps}\|\hat{z} - z\|^2 \\
  \nonumber
&= N\gamma^2 z^\top (K + N\gamma I)^{-2}z + \tfrac1N \tr(FK^2(K+N\gamma I)^{-2})\\
 \label{eq:biasvar}
&= \mathrm{bias}(K) + \mathrm{var}(K).
\end{align} 
Observe that the bias term is matrix-decreasing (in $K$) while the variance term is matrix-increasing. Since the estimator~\eqref{eq:estimator} requires expensive matrix inversions, it is common to replace $K$ in \eqref{eq:estimator} by an approximation $\tilde{K}$. If $\tilde{K}$ is constructed via \nys we have $\tilde{K}\preceq K$, and it directly follows that the variance shrinks with this substitution, while the bias increases. Denoting the predictions from $\tilde{K}$ by $\hat{z}_{\tilde{K}}$, Theorem~\ref{thm:krr} completes the picture of how using $\tilde{K}$ affects the risk.
\begin{theorem}
\label{thm:krr}
If $\tilde{K}$ is constructed via \dpp-Nystr\"om, then
\begin{align*}
\mathbb{E}_C \left[\sqrt{\calr(\hat{z}_{\tilde{K}})\over \calr(\hat{z})}\right]\le 1 + {(c+1)\over N\gamma }{e_{c+1}(K)\over e_c(K)}.
\end{align*}
\end{theorem}
Again, using~\cite{pemantle2014concentration}, we obtain bounds that hold with high probability (Appendix~\ref{append:sec:proof}).

\begin{proof}
We build on~\cite{bach2012sharp,alaoui2014fast}. Knowing that $\var(\tilde{K})\le \var(K)$ as  $\tilde{K}\preceq K$, it remains to bound the bias.
Using $K = B^\top B$ and $\tilde{K} = B^\top B_{\cdot,C}(B_{\cdot,C}^\top B_{\cdot,C})^\dagger B_{\cdot,C}^\top B$, we obtain
\begin{small}
\begin{align*}
K &- \tilde{K} = B^\top (I - B_{\cdot,C}(B_{\cdot,C}^\top B_{\cdot,C})^\dagger B_{\cdot,C}^\top) B \\
&= B^\top (U^C)^\perp ((U^C)^\perp)^\top B \preceq \|B^\top (U^C)^\perp ((U^C)^\perp)^\top B\|_F I \\
&= \sqrt{\nlsum_{i,j} (b_i^\top (U^C)^\perp ((U^C)^\perp)^\top b_j)^2} I \\
&\preceq \sqrt{(\nlsum_{i,j} \|b_i^\top (U^C)^\perp\|_2^2 \|b_j^\top (U^C)^\perp\|_2^2)} I \\
&= \nlsum_i \|b_i^\top (U^C)^\perp\|_2^2 I  = \nu_C I,
\end{align*}
\end{small}%
where $\nu_C = \sum_i \|b_i^\top (U^C)^\perp\|_2^2 \le \sum_i \|b_i^\top\|_2^2 = \text{tr}(K)$. Since $(K-\tilde{K})$ and $\nu_CI$ commute, we have
\begin{small}
\begin{align*}
\|(\tilde{K} + & N\gamma I)^{-1}(K - \tilde{K})\|_2^2  \\
& = \|(\tilde{K} + N\gamma I)^{-1}(K - \tilde{K})^2(\tilde{K} + N\gamma I)^{-1}\|_2\\
&\le \nu_C^2 \|(\tilde{K} + N\gamma I)^{-2}\|_2 \le \Big({\nu_C\over N\gamma}\Big)^2.
\end{align*}
\end{small}%
It follows that
\begin{small}
\begin{align*}
\|(\tilde{K} + &N\gamma I)^{-1}z - (K + N\gamma I)^{-1}z\|_2  \\
&= \|(\tilde{K} + N\gamma I)^{-1}(K - \tilde{K})(K + N\gamma I)^{-1}z\|_2\\
&\le \|(\tilde{K} + N\gamma I)^{-1}(K - \tilde{K})\|_2 \|(K + N\gamma I)^{-1}z\|_2\\
&\le {\nu_C\over N\gamma} \|(K + N\gamma I)^{-1}z\|_2.
\end{align*}
\end{small}%
Hence,
\begin{small}
\begin{align*}
&\sqrt{z^\top (\tilde{K} + N\gamma I)^{-2} z} = \|(\tilde{K} + N\gamma I)^{-1}z\|_2\\
&\le \|(K + N\gamma I)^{-1}z\|_2 + \|(\tilde{K} + N\gamma I)^{-1}z - (K + N\gamma I)^{-1}z\|_2\\
&\le (1 + {\nu_C\over N\gamma}) \|(K + N\gamma I)^{-1}z\|_2\\
&= (1 + {\nu_C\over N\gamma}) \sqrt{z^\top (K + N\gamma I)^{-2}z}.
\end{align*}
\end{small}%
Finally, this inequality implies that
\begin{small}
\begin{align*}
\sqrt{\mathrm{bias}(\tilde{K})\over \mathrm{bias}(K)}\le (1 + {\nu_C\over N\gamma}).
\end{align*}
\end{small}

Taking the expectation over $C\sim c$-\dpp{($K$)} yields
\begin{small}
\begin{align*}
\mathbb{E}_C\left[\sqrt{\mathrm{bias}(\tilde{K})\over \mathrm{bias}(K)}\right] \le 1 + \mathbb{E}_C\left[{\nu_C \over N\gamma}\right]= 1 + {(c+1)\over N\gamma }{e_{c+1}(K)\over e_c(K)}.
\end{align*}
\end{small}%
Together with the fact that $\mathrm{var}(\tilde{K})\le \mathrm{var}(K)$, we obtain
\begin{small}
\begin{align}
\nonumber
\mathbb{E}_C \left[\sqrt{\calr(\hat{z}_{\tilde{K}})\over \calr(\hat{z})}\right] &= \mathbb{E}_C \left[\sqrt{\mathrm{bias}(\tilde{K}) + \mathrm{var}(\tilde{K})\over \mathrm{bias}(K) + \mathrm{var}(K)}\right]\\
&\le 1 + {(c+1)\over N\gamma }{e_{c+1}(K)\over e_c(K)}
\end{align}
\end{small}
for any $k \leq c$.
\end{proof}

\paragraph{Remarks.} 
Theorem~\ref{thm:krr} quantifies how the learning results depend on the decay of the spectrum of $K$.
In particular, the ratio $e_{c+1}(K)/e_c(K)$ closely relates to the effective rank of $K$: if $\lambda_c > a$ and $\lambda_{c+1}\ll a$, this ratio is almost zero, resulting in near-perfect approximations and no loss in learning.

There exist works that consider \nys methods in this scenario~\cite{bach2012sharp,alaoui2014fast}. Our theoretical bounds could also be tightened in this setting, possibly by a tighter bound on the elementary symmetric polynomial ratio. This theoretical exercise may be worthwhile given our extensive experiments comparing \dpp-\nys against other state-of-art methods in~\refsec{sec:exp:krr} that reveal the superior performance of \dpp-\nys.

\section{Fast Mixing Markov Chain \dpp}
\label{sec:mix}
Despite its excellent empirical performance and strong theoretical results, determinantal sampling for \nys has rarely been used in applications due to the computational cost of $\co(N^3)$ for directly sampling from a \dpp, which involves an eigendecomposition. Instead, we follow a different route: an MCMC sampler, which offers a promising alternative if the chain mixes fast enough. Recent empirical results provide initial evidence \cite{kang2013fast}, but without a theoretical analysis\footnote{The analysis in \cite{kang2013fast} is not correct.}; other recent works~\citep{rebeschini2015fast,gotovos2015sampling} do not apply to our cardinality-constrained setting.
We offer a theoretical analysis that confirms fast mixing (i.e., polynomial or even \emph{linear}-time sampling) under certain conditions, and connect it to our empirical results. The empirical results in Section~\ref{sec:exp} illustrate the favorable performance of \dpp-\nys in trading off time and error. Concurrently with this paper, \citet{anari2016monte} derived a different, general analysis of fast mixing that also confirms our observations.

Algorithm~\ref{algo:mcdpp} shows a Gibbs sampler for $k$-{\dpp}. Starting with a uniformly random set $Y_0$, at iteration $t$, we try to swap an element $y^{\text{in}} \in Y_t$ with an element $y^{\text{out}} \notin Y_t$, according to $\Pr(Y_t)$ and $\Pr(Y_{t} \cup \{y^{\text{out}}\} \setminus \{y^{\text{in}}\})$. The stationary distribution of this chain is exactly the desired $k$-\dpp{($K$)}.

\begin{algorithm}
	\caption{Gibbs sampler for $c$-\dpp}\label{algo:mcdpp}
	\begin{algorithmic} 
	\STATE \textbf{Input:} $K$ the kernel matrix, $\cy = [N]$ the ground set
	\STATE \textbf{Output:} $Y$ sampled from exact $c$-\dpp{($K$)}
	\STATE Randomly Initialize $Y\subseteq \cy$, $|Y|=c$
	\WHILE{not mixed}
		\STATE Sample $b$ from uniform Bernoulli distribution
		\IF{$b = 1$}
			\STATE Pick $y^{\text{in}}\in Y$ and $y^{\text{out}}\in \cy\backslash Y$ uniformly randomly	\\
			\STATE $q(y^{\text{in}},y^{\text{out}},Y)\leftarrow{\det(K_{Y\cup \{y^{\text{out}}\}\backslash\{y^{\text{in}}\}})\over \det{K_{Y\cup \{y^{\text{out}}\}\backslash\{y^{\text{in}}\}}} + \det(K_Y)}$
			\STATE $Y\leftarrow Y\cup \{y^{\text{out}}\}\backslash\{y^{\text{in}}\}$ with prob. $q(y^{\text{in}},y^{\text{out}},Y)$	
		\ENDIF
	\ENDWHILE
	\end{algorithmic}
\end{algorithm}

The \emph{mixing time} $\tau(\eps)$ of the chain is the number of iterations until the distribution over the states (subsets) is close to the desired one, as measured by total variation: $\tau(\eps) = \min\{t|\max_{Y_0} \mathrm{TV}(Y_t,\pi) \leq \epsilon \}$. We bound $\tau(\eps)$ via coupling techniques.
Given a Markov chain $(Y_t)$ on a state space $\Omega$ with transition matrix $P$, a \emph{coupling} is a new chain $(Y_t,Z_t)$ on $\Omega\times \Omega$ such that both $(Y_t)$ and $(Z_t)$, if considered marginally, are Markov chains with the same transition matrix $P$. The key point of coupling is to construct such a new chain to encourage $Y_t$ and $Z_t$ to \emph{coalesce} quickly. If, in the new chain, $\Pr(Y_t\ne Z_t)\le \eps$ for some fixed $t$ regardless of the starting state $(Y_0,Z_0)$, then $\tau(\eps)\le t$~\cite{aldous1982some}.

Such coalescing chains can be difficult to construct. \emph{Path coupling} \cite{bubley1997path} relieves this burden by reducing the coupling to adjacent states in an appropriately constructed state graph. The coupling of arbitrary states follows by aggregation over a path between the states. Path coupling is formalized in the following lemma.

\begin{lemma} \cite{bubley1997path,dyer1998more}
\label{lem:pathcoupling}
Let $\delta$ be an integer-valued metric on $\Omega\times \Omega$ where $\delta(\cdot,\cdot)\le D$. Let $E$ be a subset of $\Omega\times\Omega$ such that for all $(Y_t,Z_t)\in\Omega\times\Omega$ there exists a path $Y_t = X_0,\ldots, X_r = Z_t$ between $Y_t$ and $Z_t$ where $(X_i,X_{i+1})\in E$ for $i\in[r-1]$ and $\sum_i \delta(X_i,X_{i+1}) = \delta(Y_t,Z_t)$. Suppose a coupling $(R,T)\to(R',T')$ of the Markov chain is defined on all pairs in $E$ such that there exists an $\alpha < 1$ such that $\mathbb{E}[\delta(R',T')]\le \alpha \delta(R,T)$ for all $(R,T)\in E$, then we have 
\begin{align*}
\tau(\eps)\le {\log(D\eps^{-1})\over (1 - \alpha)}.
\end{align*}
\end{lemma}
The lemma says that if we have a contraction of the two chains in expectation ($\alpha < 1$), then the chain mixes fast.
With the path coupling lemma, we obtain a bound on the mixing time that can be \emph{linear} in the data set size $N$.

The actual mixing time depends on three quantities that relate to how sensitive the transition probabilities are to swapping a single element in a set of size $c$. Consider an arbitrary set $S$ of columns, $|S|=c-1$, and complete it to two $c$-sets $R = S \cup \{r\}$ and $T = S \cup \{t\}$ that differ in exactly one element. Our quantities are, for $ u \notin R \cup T$, and $v \in S$:
\begin{align*}
  p_1(S,r,t,u) &= \min\{q(r,u,R),q(t,u,T)\} \\
  p_2(S,r,t,u) &= \min\{q(v,t,R),q(v,u,T)\} \\
  p_3(S,r,t,v,u)&=|q(v,u,R) - q(v,u,T)|.
\end{align*}

\begin{theorem}
\label{thm:mix}
Let the contraction coefficient $\alpha$ be given by
\begin{small}
\begin{align*}
\alpha = &\max_{|S| = c-1,r,t\in[n]\backslash S,r\neq t}\sum_{u_3\in S, u_4\notin S\cup\{r,t\}}p_3(S,r,t,u_3,u_4)- \sum_{u_1\notin S\cup\{r,t\}}p_1(S,r,t,u_1)-\sum_{u_2\in S}p_2(S,r,t,u_2).
\end{align*}
\end{small}%
When $\alpha < 1$, the mixing time for the Gibbs sampler in Algorithm~\ref{algo:mcdpp} is bounded as
\begin{align*}
\tau(\eps) \le {2c(N-c)\log (c \eps^{-1})\over (1 - \alpha)}.
\end{align*}
\end{theorem}
\begin{proof}
We bound the mixing time via path coupling.
Let $\delta(R,T) = |R\oplus T|/2$ be half the Hamming distance on the state space, and define $E$ to consist of all state pairs $(R,T)$ in $\Omega\times\Omega$ such that $\delta(R,T) = 1$. We intend to show that for all states $(R,T)\in E$ and next states $(R',T')\in E$, we have $\mathbb{E}[\delta(R',T')]\le \alpha \delta(R,T)$ for an appropriate $\alpha$.

Since $\delta(R,T) = 1$, the sets $R$ and $T$ differ in only two entries. Let $S = R\cap T$, so $|S| = c-1$ and $R = S\cup \{ r\}$ and $T = S\cup\{t\}$. For a state transition, we sample an element $\rin\in R$ and $\rout\in[n]\backslash R$ as switching candidates for $R$,  and elements $\tin\in T$ and $\tout\in[n]\backslash T$ as switching candidates for $T$. Let $b_R$ and $b_T$ be the Bernoulli random variables indicating whether we try to make a transition. In our coupling we always set $b_R = b_T$. Hence, if $b_R = 0$ then both chains will not transition and the distance of states remains. For $b_R = b_T = 1$, we distinguish four cases: 
\paragraph{Case C1} If $\rin = r$ and $\rout = t$, we let $\tin = t$ and $\tout = r$. As a result, $\delta(R',T') = 0$.
\paragraph{Case C2} If $\rin = r$ and $\rout = u_1 \notin S\cup\{r,t\}$, we let $\tin = t$ and $\tout = u_1$. In this case, if both chains transition, then the resulting distance is zero, otherwise it remains one. With probability $p_1(S,r,t,u_1) = \min\{q(r,u_1,R),q(t,u_1,T)\}$ both chains transition.
\paragraph{Case C3} If $\rin = u_2\in S$ and $\rout = t$, we let $\tin = u_2$ and $\tout = r$. Again, if both chains transition, then the resulting distance is $\delta(R',T')= 0$, otherwise it remains one. With probability $p_2(S,r,t,u_2) = \min\{q(u_2,t,R),q(u_2,u_1,T)\}$ both chains transition.
\paragraph{Case C4} If $\rin = u_3\in S$ and $\rout = u_4\notin S\cup\{r,t\}$, we let $\tin = u_3$ and $\tout = u_4$. If both chains make the same transition (both move or do not move), the resulting distance is one, otherwise it increases to 2. The distance increases with probability $p_3(S,r,t,u_3,u_4)=|q(u_3,u_4,R) - q(u_3,u_4,T)|$.

With those four cases, we can now bound $\mathbb{E}[\delta(R',T')]$. For all $(R,T) \in E$, i.e., $\delta(R,T)=1$:
\begin{small}
\begin{align*}
&{\mathbb{E}[\delta(R',T')]\over \mathbb{E}[\delta(R,T)]} = {1\over 2} + \text{Pr}(C2) \mathbb{E}[\delta(R',T')| C2] + \text{Pr}(C3) \mathbb{E}[\delta(R',T')| C3] + \text{Pr}(C4) \mathbb{E}[\delta(R',T')| C4]\\
&= \frac12 + {1\over 2c(n-c)}\big(\sum_{u_1\notin S\cup\{r,t\}}(1 - p_1(u_1)) + \sum_{u_2\in S}(1 - p_2(u_2)) + 
\sum_{\substack{u_3\in S,\\ u_4\notin S\cup\{r,t\}}}(1 + p_3(u_3,u_4))\big) \\
&= {1\over 2c(n-c)}\big(2c(n-1)+\sum_{\substack{u_3\in S,\\ u_4\notin S\cup\{r,t\}}}p_3(u_3,u_4) - \sum_{u_1\notin S\cup\{r,t\}}p_1(u_1)-\sum_{u_2\in S}p_2(u_2) -1\big),
\end{align*}
\end{small}%
where we did not explicitly write the arguments $S,r,t$ to $p_{1}, p_2, p_3$.
For
\begin{small}
\begin{align*}
\alpha = \max_{\substack{|S| = c-1,\\r,t\in[n]\backslash S,\\r\neq t}}&\sum_{\substack{u_3\in S,\\ u_4\notin S\cup\{r,t\}}}p_3(u_3,u_4)-\sum_{u_1\notin S\cup\{r,t\}}p_1(u_1)-\sum_{u_2\in S}p_2(u_2)
\end{align*}
\end{small}%
and $\alpha < 1$ the Path Coupling Lemma~\ref{lem:pathcoupling} implies that
\begin{align*}
\tau(\eps) &\le {2c(N-c)\log (c \eps^{-1})\over (1 - \alpha)}. \qedhere
\end{align*}
\end{proof}

\paragraph{Remarks.} If $\alpha<1$ is fixed, then the mixing time (running time) depends only linearly on $N$. The coefficient $\alpha$ itself depends on our three quantities. In particular, fast mixing requires $p_3$ (the difference between transition probabilities) to be very small compared to $p_1$, $p_2$, at least on average. The difference $p_3$ measures how exchangeable two points $r$ and $t$ are. This notion of symmetry is closely related to a symmetry that determines the complexity of submodular maximization~\citep{vondrak13} (indeed, $F(S)=\log\det K_S$ is a submodular function). This symmetry only needs to hold for most pairs $r$, $t$, and most swapping points $u$, $v$. It holds for kernels with sufficiently fast-decaying similarities, similar to the conditions in~\cite{rebeschini2015fast} for unconstrained sampling.

One iteration of the sampler can be implemented efficiently in $\co(c^2)$ time using block inversion \cite{golub2012matrix}. Additional speedups via quadrature are also possible \cite{li16icmlquad}. Together with the analysis of mixing time, this leads to fast sampling methods for $k$-{\dpp}s.

\section{Experiments}
\label{sec:exp}
In our experiments, we evaluate the performance of \dpp-\nys on both kernel approximation and kernel learning tasks, in terms of running time and accuracy.

We use 8 datasets: Abalone, Ailerons, Elevators, CompAct, CompAct(s), Bank32NH, Bank8FM and California Housing\footnote{\url{http://www.dcc.fc.up.pt/~ltorgo/Regression/DataSets.html}}. We subsample 4,000 points from each dataset (3,000 training and 1,000 test).
Throughout our experiments, we use an RBF kernel and choose the bandwidth $\sigma$ and regularization parameter $\lambda$ for each dataset by 10-fold cross-validation. We initialize the Gibbs sampler via Kmeans++ and run for 3,000 iterations. Results are averaged over 3 random subsets of data.

\subsection{Kernel Approximation}
\label{sec:exp:app}

\begin{figure}
\begin{center}
\includegraphics[width=0.8\textwidth]{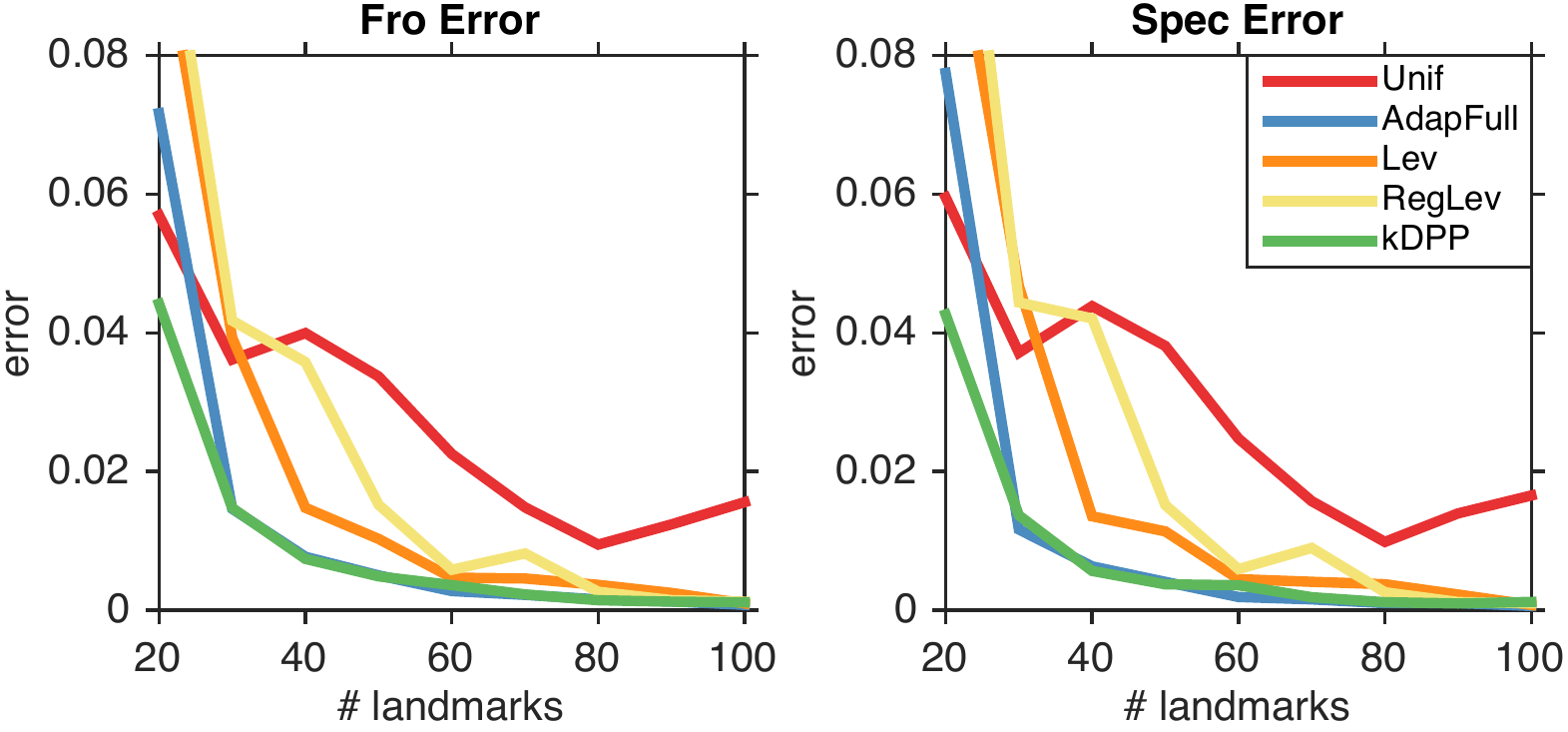}
\caption{\small Relative Frobenius/spectral norm errors from different kernel approximations (Ailerons data).}
\label{fig:app_ailerons_mc}
\end{center}
\end{figure}

We first explore \dpp-\nys (\ours in the figures) for approximating kernel matrices. We compare to uniform sampling~(\unif) and leverage score sampling (\lev) \cite{gittens2013revisiting} as baseline landmark selection methods. We also include AdapFull~(\adap) \cite{deshpande2006matrix} that performs quite well in practice but scales poorly, as $\co(N^2)$, with the size of dataset. Although sampling with regularized leverage scores (\rlev) \cite{alaoui2014fast} is not originally designed for kernel approximations, we include its results to see how regularization affects leverage score sampling. 

Figure~\ref{fig:app_ailerons_mc} shows example results on the Ailerons data; further results may be found in the appendix. \dpp-\nys performs well, achieving the lowest error as measured in both spectral and Frobenius norm. The only method that is on par in terms of accuracy is \adap, which has a much higher running time.

\begin{figure}
\begin{center}
\includegraphics[width=0.8\textwidth]{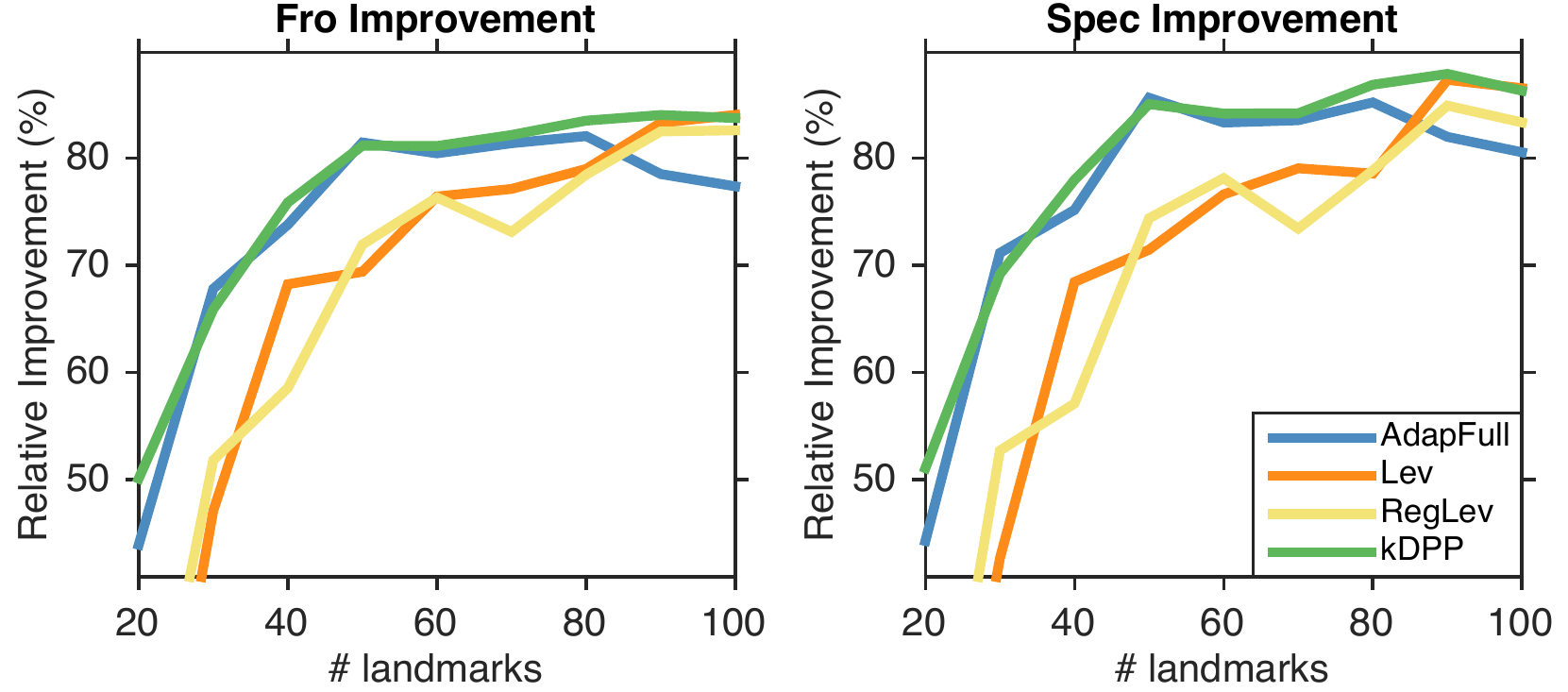}
\caption{\small Improvement in relative Frobenius/spectral norm errors~(\%) over \unif~(with corresponding landmark sizes) for kernel approximation, averaged over all datasets. }
\label{fig:rel_app_mc}
\end{center}
\end{figure}

For a different perspective, Figure~\ref{fig:rel_app_mc} shows the improvement in error over \unif. Relative improvements are averaged over all data sets. Again, the performance of \dpp-\nys almost always dominate those of other methods, and achieves an up to 80\% reduction in error.

\begin{figure}
\begin{center}
\includegraphics[width=0.8\textwidth]{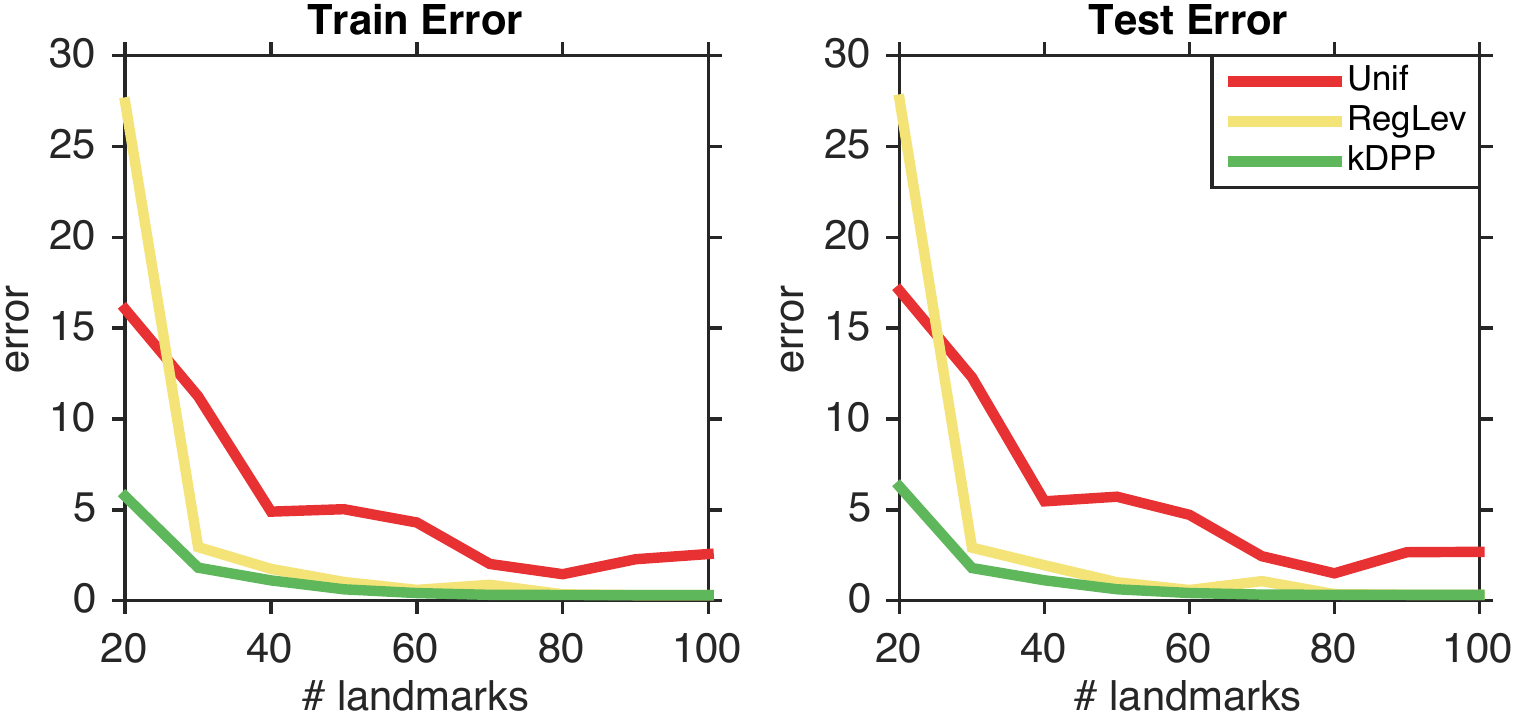}
\caption{\small Training and test errors for kernel ridge regression with different \nys approximations (Ailerons data).}
\label{fig:krr_ailerons}
\end{center}
\end{figure}

\begin{figure}
\begin{center}
\includegraphics[width=0.8\textwidth]{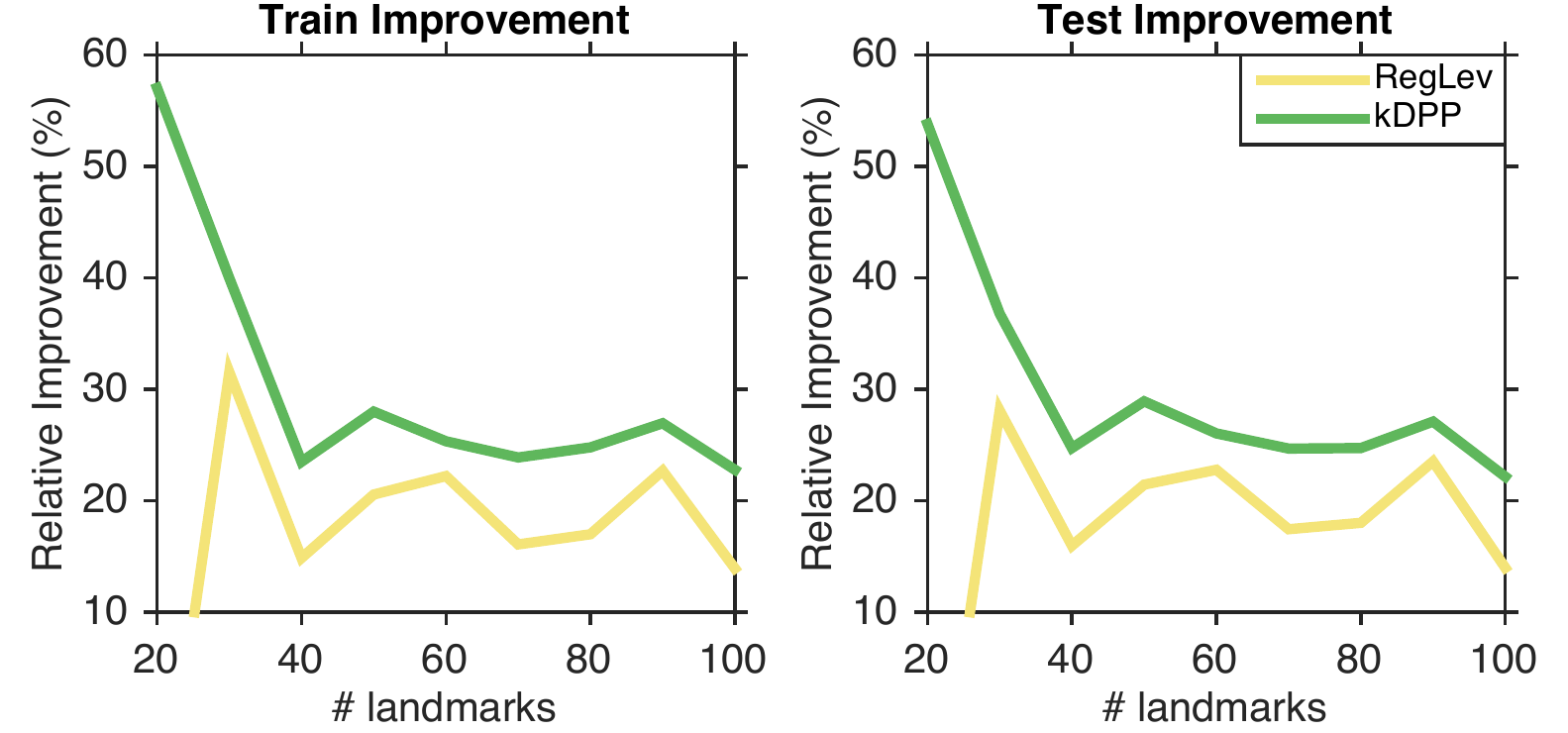}
\caption{\small Improvements in training/test errors~(\%) over uniform sampling (with same number of landmarks) in kernel ridge regression, averaged over all datasets.}
\label{fig:rel_krr}
\end{center}
\end{figure}

\subsection{Kernel Ridge Regression}
\label{sec:exp:krr}
Next, we apply \dpp-\nys to kernel ridge regression, comparing against uniform sampling (\unif) \citep{bach2012sharp} and regularized leverage score sampling (\rlev) \citep{alaoui2014fast} which have theoretical guarantees for this task. 
Figure~\ref{fig:krr_ailerons} illustrates an example result: non-uniform sampling greatly improves accuracy, with \ours improving over regularized leverage scores in particular for a small number of landmarks, where a single column has a larger effect.

Figure~\ref{fig:rel_krr} displays the average improvement over \unif, averaged over 8 data sets. Again, the performance of \ours dominates \rlev and \unif, and leads to gains in accuracy. On average \ours consistently achieves more than $20\%$ improvement over \unif.

\begin{figure}
\centering
	\begin{subfigure}{.4\textwidth}
	\centering
	\includegraphics[width=\textwidth]{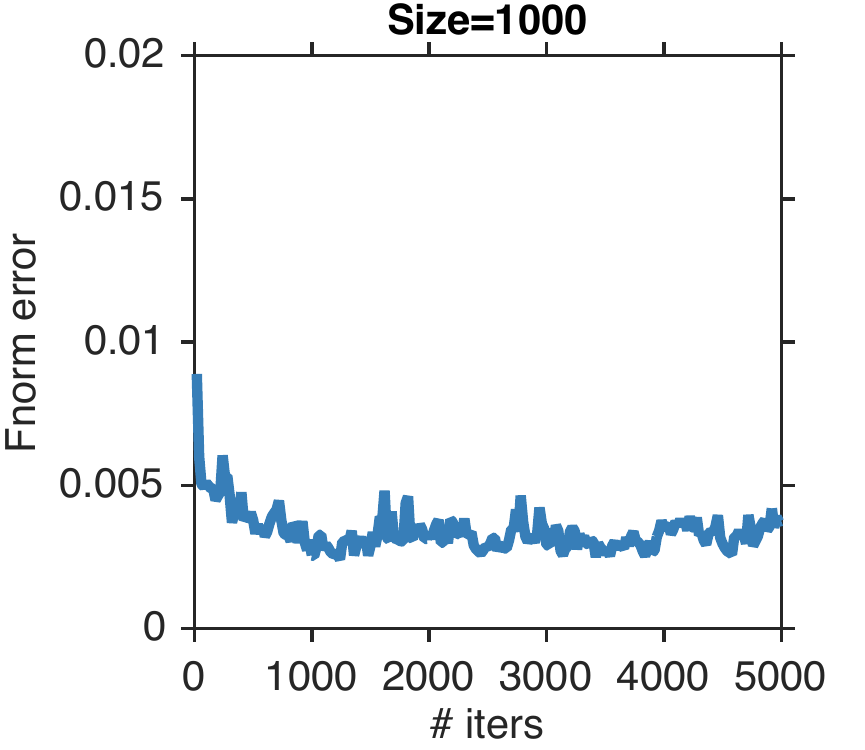}
	\end{subfigure}%
	\begin{subfigure}{.4\textwidth}
	\centering
	\includegraphics[width=\textwidth]{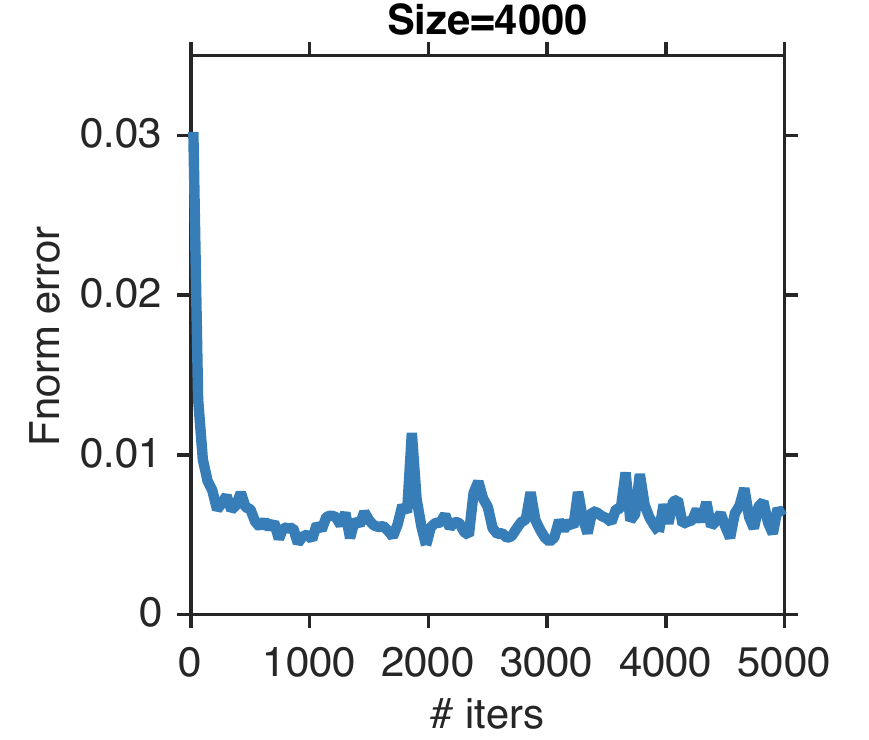}
	\end{subfigure}
\caption{\small Relative Frobenius norm error of \dpp-\nys with 50 landmarks as changing across iterations of the Markov Chain (Ailerons data).}
\label{fig:conv_ailerons_50_fnorm}
\end{figure}

\subsection{Mixing of the Gibbs Markov Chain}
\label{sec:exp:conv}

In the next experiment, we empirically study the mixing of the Gibbs chain with respect to matrix approximation errors, the ultimate measure that is of interest in our application of the sampler.
We use $c=50$ and choose $N$ as 1,000 and 4,000. To exclude impacts of the initialization, we pick the initial state $Y_0$ uniformly at random. We run the chain for 5,000 iterations, monitoring how the error changes with the number of iterations. Example results on the Ailerons data are shown in Figure~\ref{fig:conv_ailerons_50_fnorm}. Empirically, the error drops very quickly and afterwards fluctuates only little, indicating a fast convergence of the approximation error. Other error measures and larger $c$, included in the appendix, confirm this trend. 

Notably, our empirical results suggest that the mixing time does not increase much as $N$ increases greatly, suggesting that the Gibbs sampler remains fast even for large $N$.

In Theorem~\ref{thm:mix}, the mixing time depends on the quantity $\alpha$. By subsampling 1,000 random sets $S$ and column indices $r,t$, we approximately computed $\alpha$ on our data sets. We find that, as expected, $\alpha < 1$ in particular for kernels with a smaller bandwidth, and in general $\alpha$ increases with $k$. In accordance with the theory, we found that the mixing time (in terms of error) too increases with $k$. In practice, we observe a fast drop in error even for cases where $\alpha > 1$, indicating that Theorem~\ref{thm:mix} is conservative and that the iterative MCMC approach is even more widely applicable.

\subsection{Time-Error Tradeoffs}

\begin{figure}[h!]
\centering
	\begin{subfigure}{.4\textwidth}
	\centering
	\includegraphics[width=\textwidth]{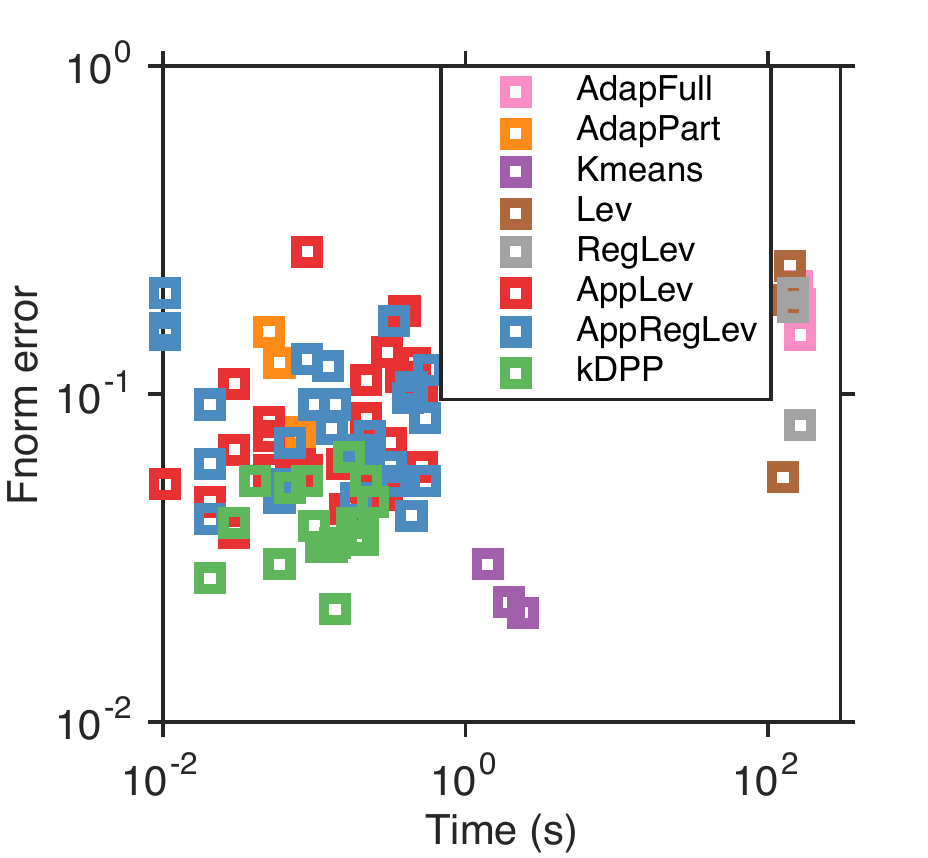}
	\end{subfigure}%
	\begin{subfigure}{.4\textwidth}
	\centering
	\includegraphics[width=\textwidth]{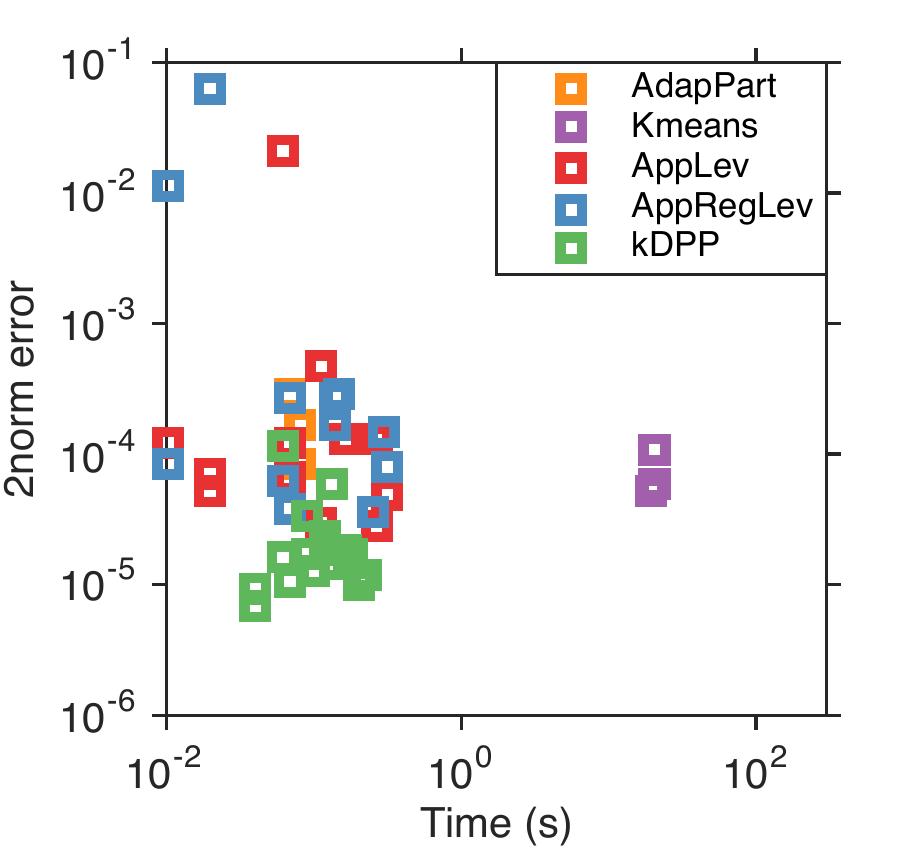}
	\end{subfigure}
	\caption{Time-Error tradeoffs with 20 landmarks on Ailerons (size 4,000) and California Housing (size 12,000). Time and Errors are shown on a log scale. Bottom left is the best (low error, low running time), top right is the worst. We did not include \adap, \lev and \rlev on California Housing due to their long running times.}
	\label{fig:tradeoff}
\end{figure}

Iterative methods like the Gibbs sampler offer tradeoffs between time and error. The longer the Markov Chain runs, the closer the sampling distribution is to the desired \dpp, and the higher the accuracy obtained by Nystr\"om. We hence explicitly show the time and accuracy trade-off of the sampler on Ailerons (of size 4,000) for up to 200 and California Housing (of size 12,000)  for up to 100 iterations.

A similar tradeoff occurs with leverage scores. For the experiments in the other sections, we computed the (regularized) leverage scores for \lev and \rlev exactly. This requires a full, computationally expensive eigendecomposition.
For a fast, rougher approximation, we here compare to an approximation mentioned in~\cite{alaoui2014fast}. Concretely, we sample $p$ elements with probability proportional to the diagonal entries of kernel matrices $K_{ii}$, and then use a \nys-like method to construct an approximate low-rank decomposition of $K$, and compute scores based on this approximation. We vary $p$ from 20 to 340 on Ailerons and 20 to 140 on California Housing to show the tradeoff for approximate leverage score sampling~(\applev) and regularized leverage score sampling~(\apprlev). We also include AdapPartial (\adappart)~\cite{kumar2012sampling} that approximates \adap and is much more efficient, and Kmeans \nys (\kmeans)~\cite{zhang2008improved} that empirically perform very well in kernel approximation.

Figure~\ref{fig:tradeoff} summarizes and compares the tradeoffs offered by these different methods on the Ailerons and California Housing datasets. The $x$ axis indicates time, the $y$ axis error, so the lower left is the preferred corner. We see that \adap, \lev and \rlev are expensive and perform worse than \ours. The approximate variants \adappart, \applev and \apprlev have comparable efficiency but higher error. On the smaller data, \kmeans is accurate but needs more time than \ours, while on the larger data it is dominated in both accuracy and time by \ours. Overall, on the larger data, \dpp-\nys offers the best tradeoff of accuracy and efficiency.

\section{Conclusion}

In this paper, we revisited the use of $k$-Determinantal Point Processes for sampling good landmarks for the \nys method. We theoretically and empirically observe its competitive performance, for both matrix approximation and ridge regression, compared to state-of-the-art methods.

To make this accurate method scalable to large matrices, we consider an iterative approach, and analyze it theoretically as well as empirically. Our results indicate that the iterative approach, a Gibbs sampler, achieves good landmark samples quickly; under certain conditions even in a number of iteratons linear in $N$, for an $N$ by $N$ matrix. Finally, our empirical results demonstrate that among state-of-the-art methods, the iterative sampler yields the best tradeoff between efficiency and accuracy.

\subsection*{Acknowledgements} This research was partially supported by an NSF CAREER award 1553284, NSF grant IIS-1409802, and a Google Research Award. We also thank Xixian Chen for discussions.

\bibliographystyle{abbrvnat}
\setlength{\bibsep}{0pt}

\newpage
\begin{appendix}

\section{Bounds that hold with High Probability}
\label{append:sec:proof}
To show high probability bounds we employ concentration results on homogeneous strongly Rayleigh measures. Specifically, we use the following theorem.
\begin{theorem}[\protect{\citet{pemantle2014concentration}}]\label{thm:concentration}
Let $\mathbb{P}$ be a $k$-homogeneous strongly Rayleigh probability measure on $\{0,1\}^N$ and $f$ an $\ell$-Lipschitz function on $\{0,1\}^N$, then
\begin{equation*}
\mathbb{P}(f - \mathbb{E}[f]\ge a\ell) \le \exp\{-a^2/8k\}.
\end{equation*}
\end{theorem}

It is known that a $k$-\dpp is a homogeneous strongly Rayleigh measure on $\{0,1\}^N$~\cite{borcea2009negative,anari2016monte}, thus Theorem~\ref{thm:concentration} applies to results obtained with $k$-\dpp. Concretely, for the bound in Theorem~\ref{thm:nys} that holds in expectation, we have the following bound that holds with high probability:

\begin{corollary}\label{cor:prob}
  When sampling $C\sim k$-\dpp{($K$)}, for any $\delta\in(0,1)$, with probability at least $1-\delta$ we have
\begin{small}
\begin{align*}
	&{\|K - K_{\cdot C} (K_{C,C})^\dagger K_{C\cdot}\|_F \over \|K - K_k\|_F} \le \left(\frac{c+1}{c+1-k}\right)\sqrt{N-k} + \sqrt{8c\log(1/\delta)}\sqrt{\sum_{i=1}^N \lambda_i^2 \over \sum_{i=k+1}^N \lambda_i^2}, \\    
    &{\|K - K_{\cdot C} (K_{C,C})^\dagger K_{C\cdot}\|_2 \over \|K - K_k\|_2} \le \left({c+1\over c+1-k}\right)(N-k) + \sqrt{8c\log(1/\delta)}\tfrac{\lambda_1}{\lambda_{k+1}},
\end{align*}
\end{small}
where $\lambda_1 \geq \lambda_2 \geq \ldots \geq \lambda_N$ are the eigenvalues of $K$. 

\begin{proof}
The Lipschitz constants of the relative errors are upper bounded by $\sqrt{\sum_{i=1}^N \lambda_i^2 \over \sum_{i=k+1}^N \lambda_i^2}$ and ${\lambda_1\over \lambda_{k+1}}$, respectively. Applying Theorem~\ref{thm:concentration} yields the results. 

\end{proof}
\end{corollary}

For the bound in Theorem~\ref{thm:krr} that holds in expectation, we have the following bound that holds with high probability:

\begin{corollary}
If $\tilde{K}$ is constructed via \dpp-\nys, then with probability at least $1 - \delta$, $\sqrt{\mathrm{bias}(\tilde{K})\over \mathrm{bias}(K)}$ is upper-bounded by
\begin{small}
\begin{equation*}
1 + {1\over  N\gamma} \left({(c+1) e_{c+1}(K)\over e_c(K)} + \sqrt{8c\log(1/\delta)} \text{tr}(K)\right).
\end{equation*}
\end{small}

\begin{proof}
Consider the function $f_C(K) = \nu_C = \sum_i \|b_i^\top (U^C)^\perp\|_2^2\le \sum_i \|b_i^\top\|_2^2 = \text{tr}(K)$. Since $0\le f_C(K)\le \text{tr}(K)$, it follows that the Lipschitz constant for $f_C$ is at most $\text{tr}(K)$. Thus when $C\sim k$-\dpp and $\delta\in(0,1)$, by applying Theorem~\ref{thm:concentration} we see that the inequality $\nu_C \le \mathbb{E}\left[\nu_C\right] + \sqrt{8c\log(1/\delta)} \text{tr}(K)$ holds with probability at least $1 - \delta$. Hence
\begin{align*}
&\mathbb{E}_C\left[\sqrt{\mathrm{bias}(\tilde{K})\over \mathrm{bias}(K)}\right]\le 1 + \mathbb{E}\left[{\nu_C\over N\gamma}\right] + \sqrt{8c\log(1/\delta)} {\text{tr}(K)\over N\gamma}\\
&\;\;\;\;= 1 + {1\over  N\gamma} \left({(c+1) e_{c+1}(K)\over e_c(K)} + \sqrt{8c\log(1/\delta)} \text{tr}(K)\right)
\end{align*}
holds with probability at least $1 - \delta$.
\end{proof}
\end{corollary}

\section{Supplementary Experiments}

\subsection{Kernel Approximation}

\reffig{append:fig:app_others_mc} shows the matrix norm relative error of various methods in kernel approximation on the remaining 7 datasets mentioned in the main text.

\begin{figure}[h!]
\centering
    \begin{subfigure}{.5\textwidth}
    \centering
    \includegraphics[width=\textwidth]{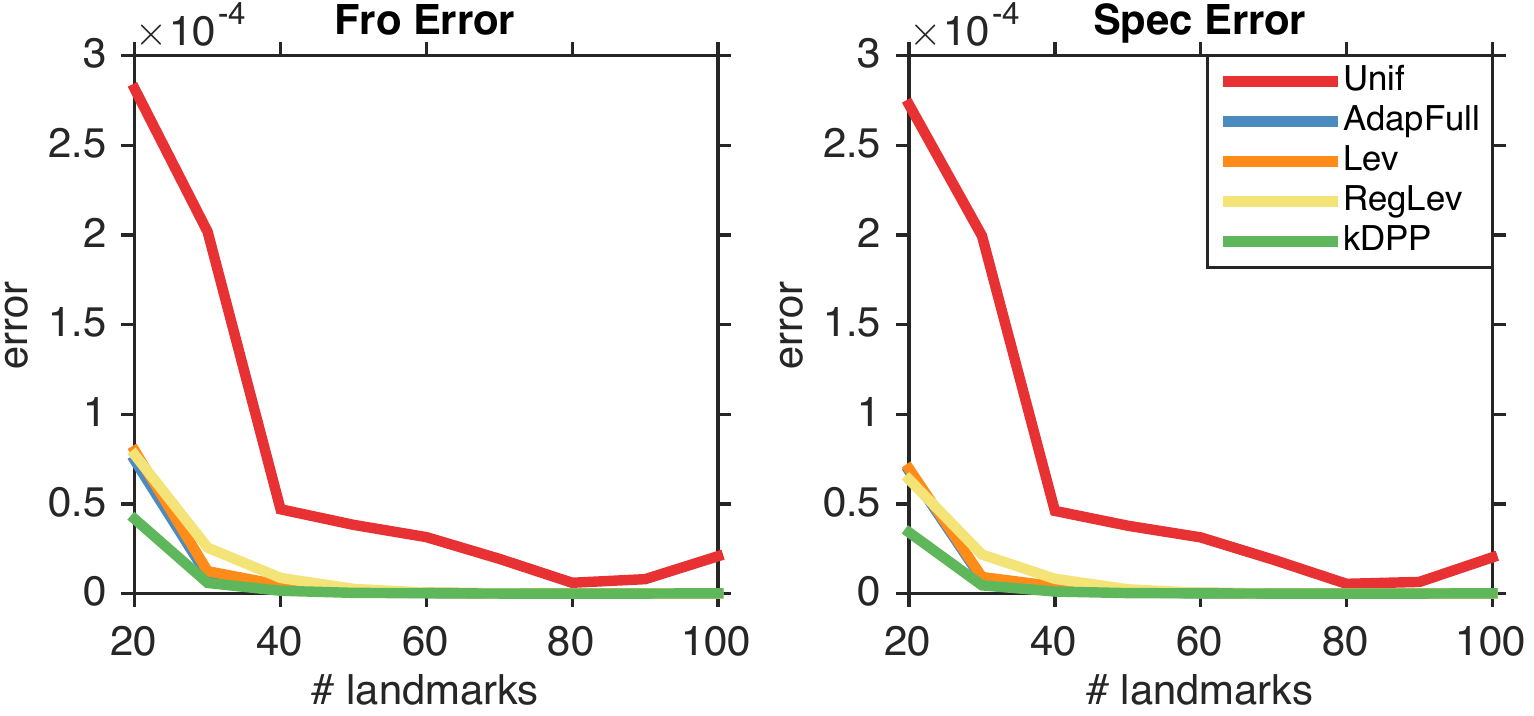}
    \caption{Abalone}
    \end{subfigure}%
    \begin{subfigure}{.5\textwidth}
    \centering
    \includegraphics[width=\textwidth]{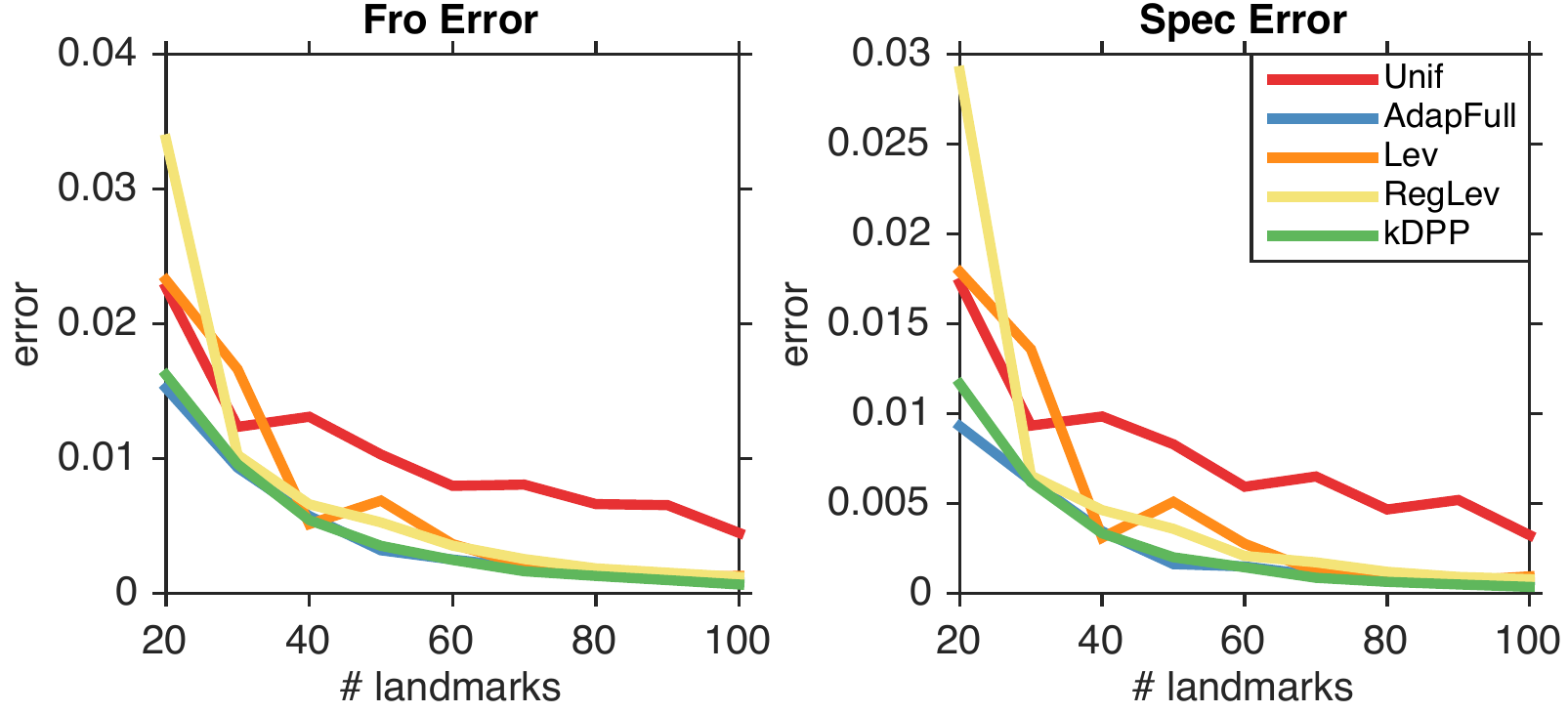}
    \caption{Bank8FM}
    \end{subfigure}
    
    \begin{subfigure}{.5\textwidth}
    \centering
    \includegraphics[width=\textwidth]{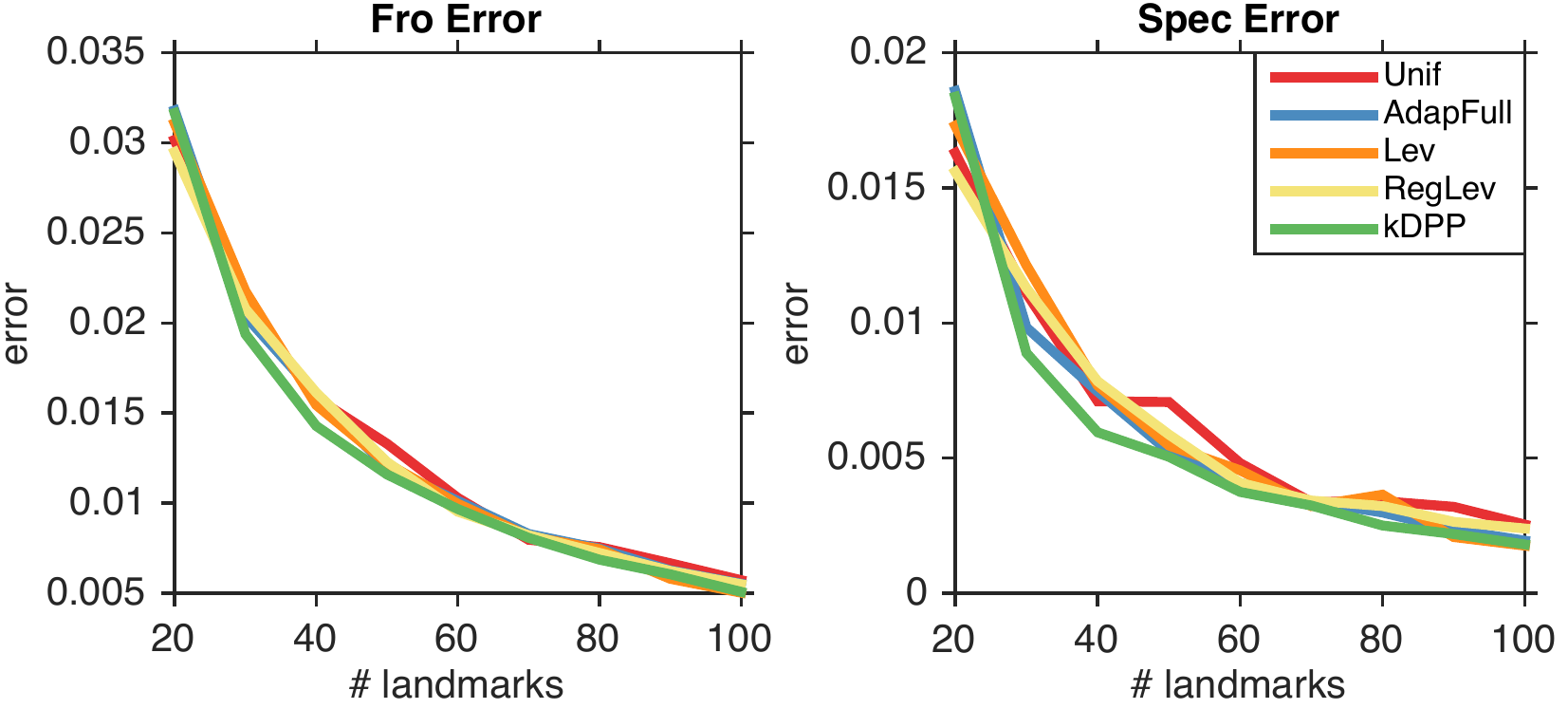}
    \caption{Bank32NH}
    \end{subfigure}%
    \begin{subfigure}{.5\textwidth}
    \centering
    \includegraphics[width=\textwidth]{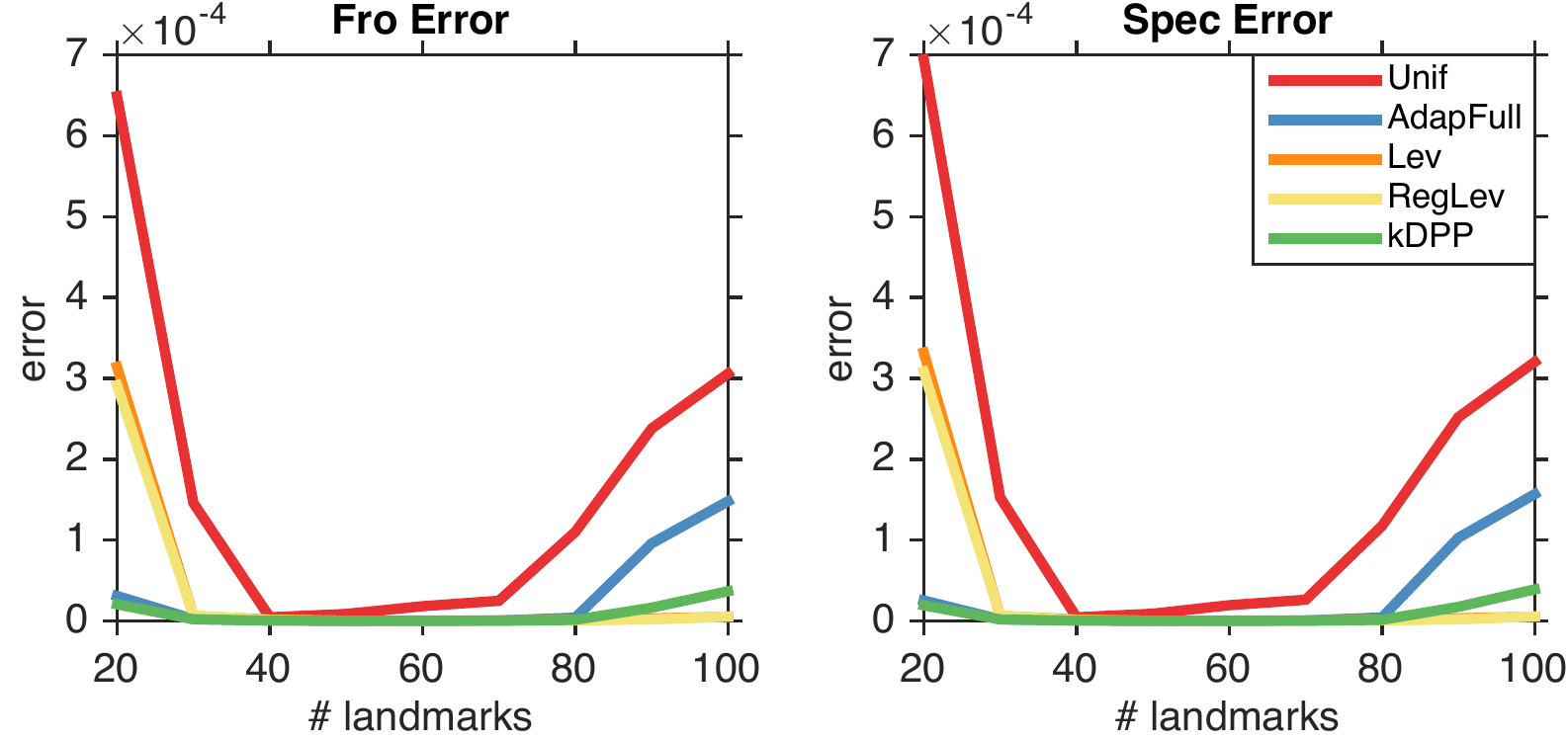}
    \caption{California Housing}
    \end{subfigure}
    
    \begin{subfigure}{.5\textwidth}
    \centering
    \includegraphics[width=\textwidth]{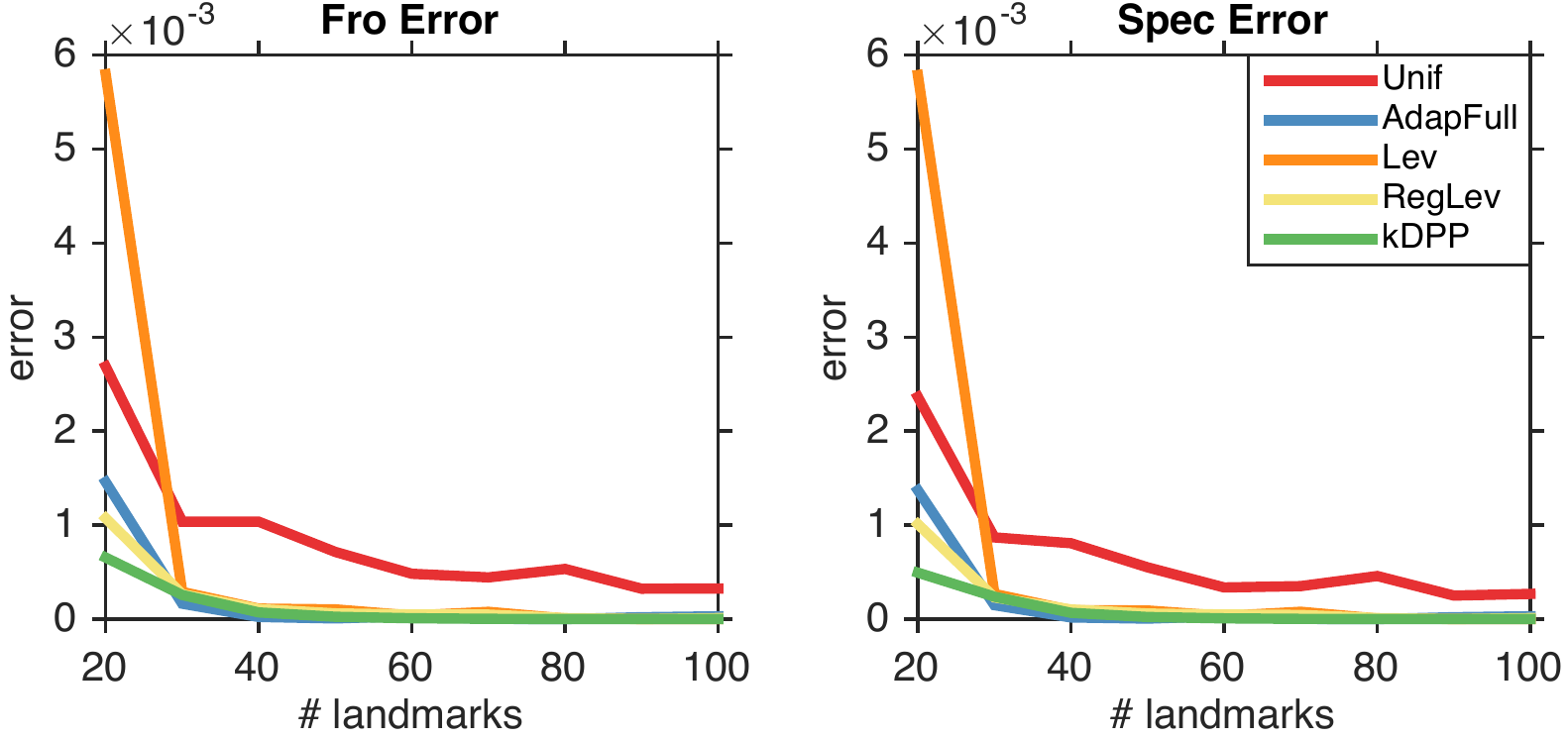}
    \caption{CompAct}
    \end{subfigure}%
    \begin{subfigure}{.5\textwidth}
    \centering
    \includegraphics[width=\textwidth]{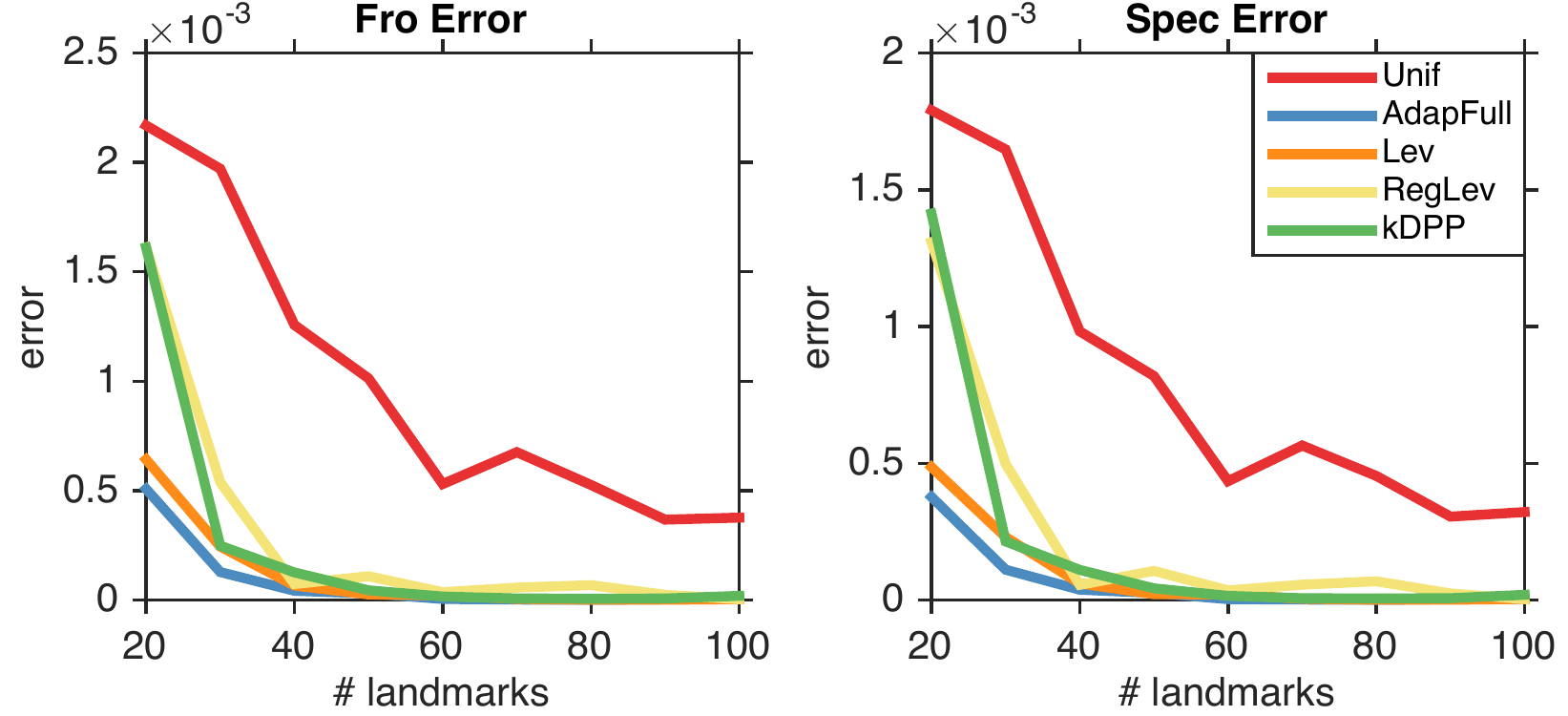}
    \caption{CompAct(s)}
    \end{subfigure}
    
    \begin{subfigure}{.5\textwidth}
    \centering	
    \includegraphics[width=\textwidth]{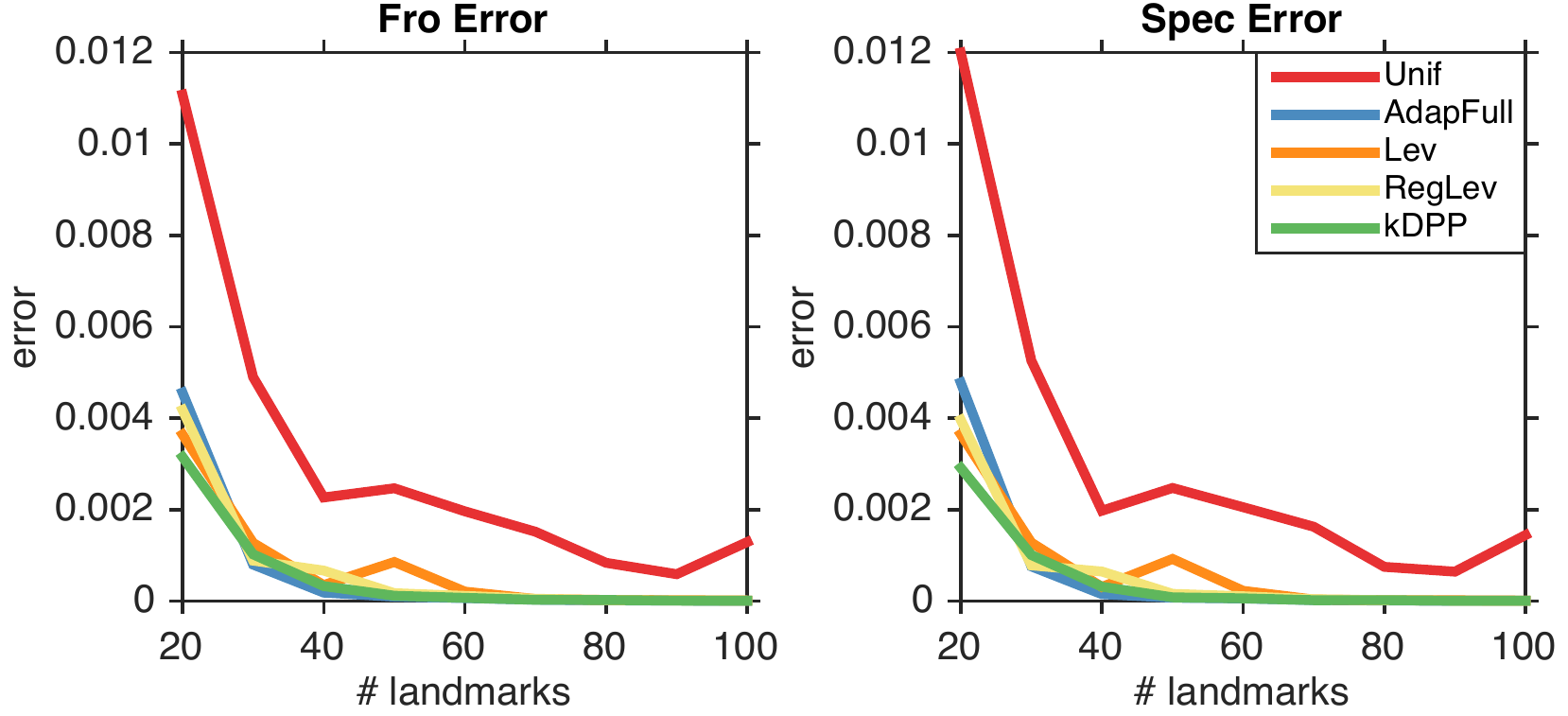}
    \caption{Elevators}
    \end{subfigure}
\caption{Relative Frobenius norm and spectral norm error achieved by different kernel approximation algorithms on the remaining 7 data sets.}
\label{append:fig:app_others_mc}
\end{figure}

\subsection{Approximated Kernel Ridge Regression}

\reffig{append:fig:krr_others_mc} shows the training and test error of various methods for kernel ridge regression on the remaining 7 datasets.

\begin{figure}[h!]
\centering
    \begin{subfigure}{.5\textwidth}
    \centering
    \includegraphics[width=\textwidth]{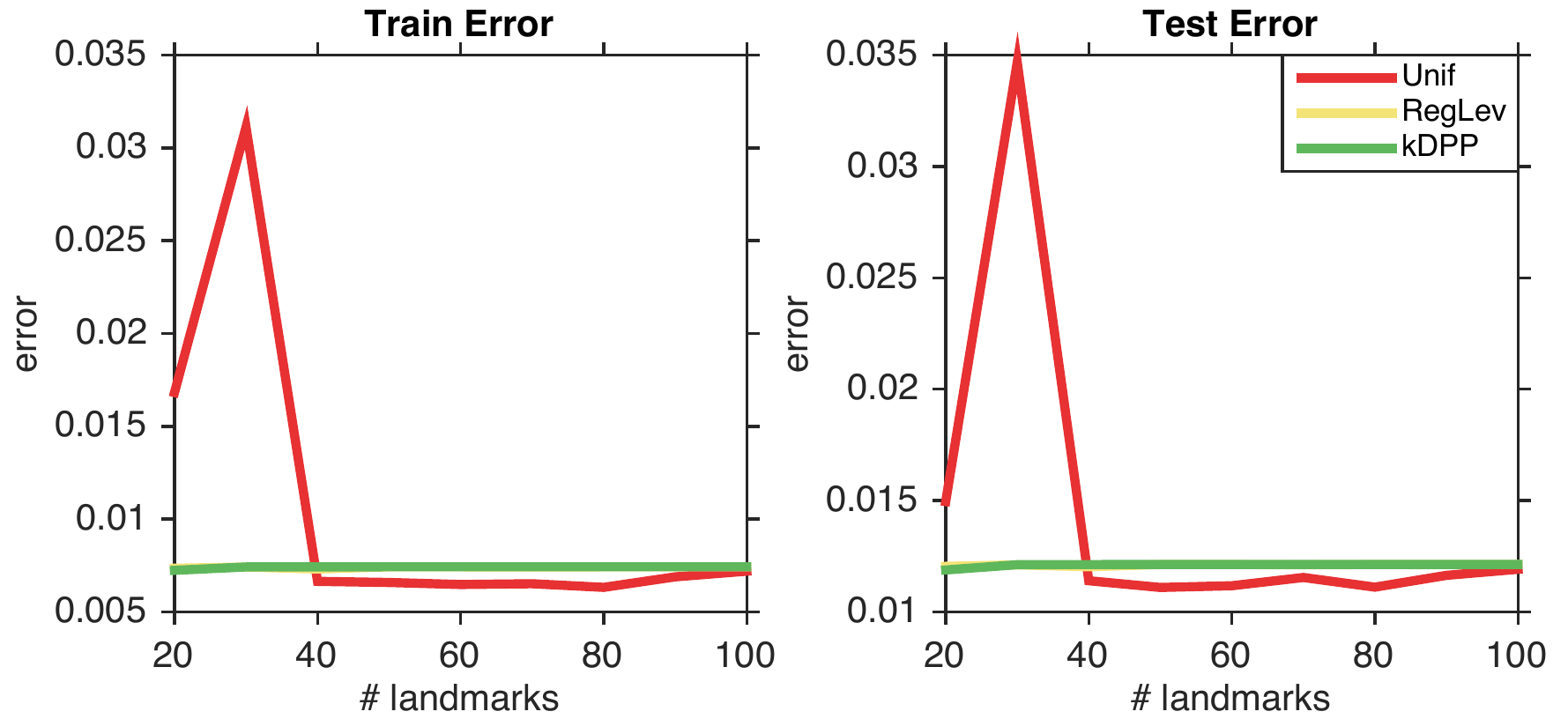}
    \caption{Abalone}
    \end{subfigure}%
    \begin{subfigure}{.5\textwidth}
    \centering
    \includegraphics[width=\textwidth]{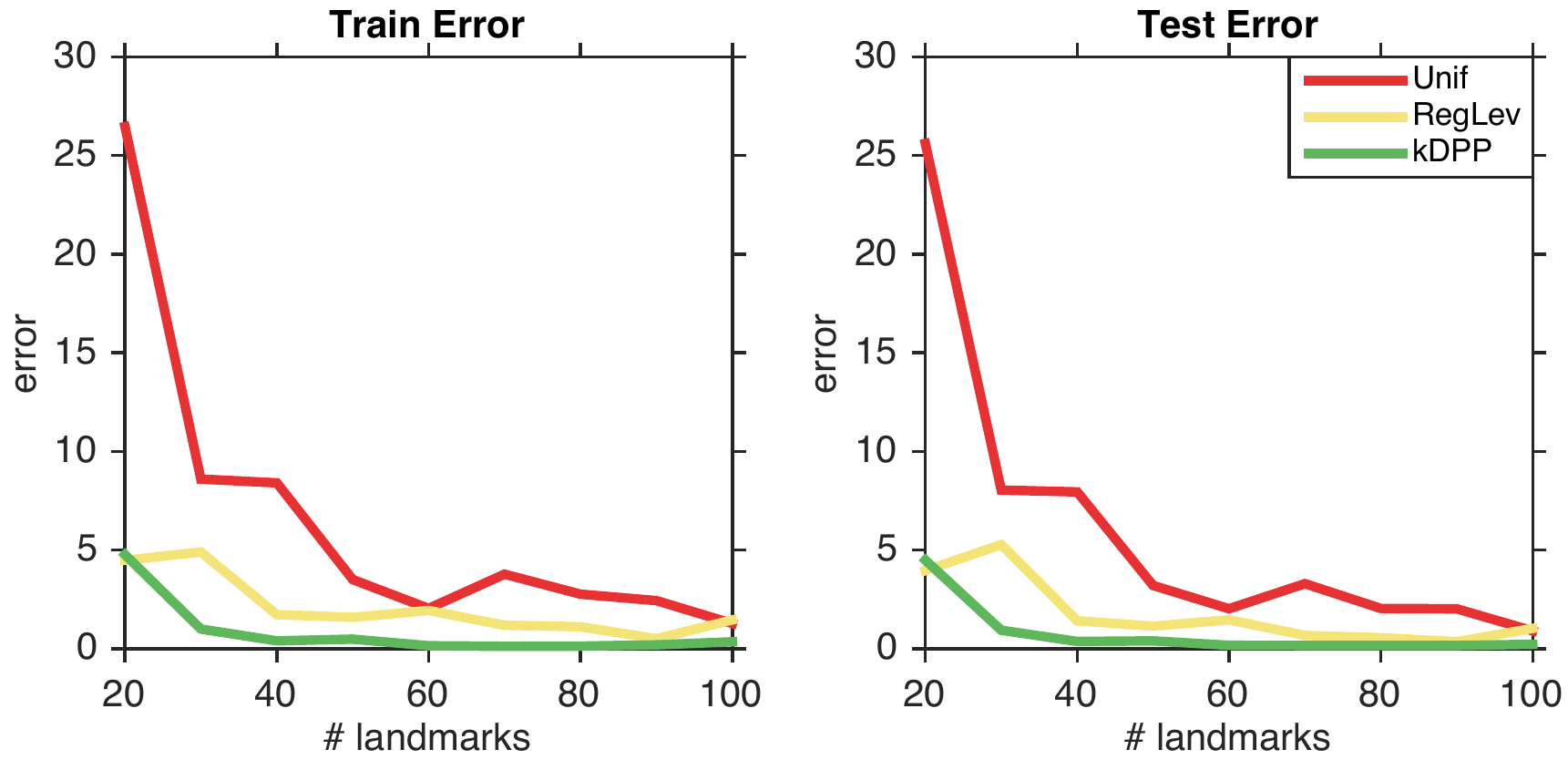}
    \caption{Bank8FM}
    \end{subfigure}
    
    \begin{subfigure}{.5\textwidth}
    \centering
    \includegraphics[width=\textwidth]{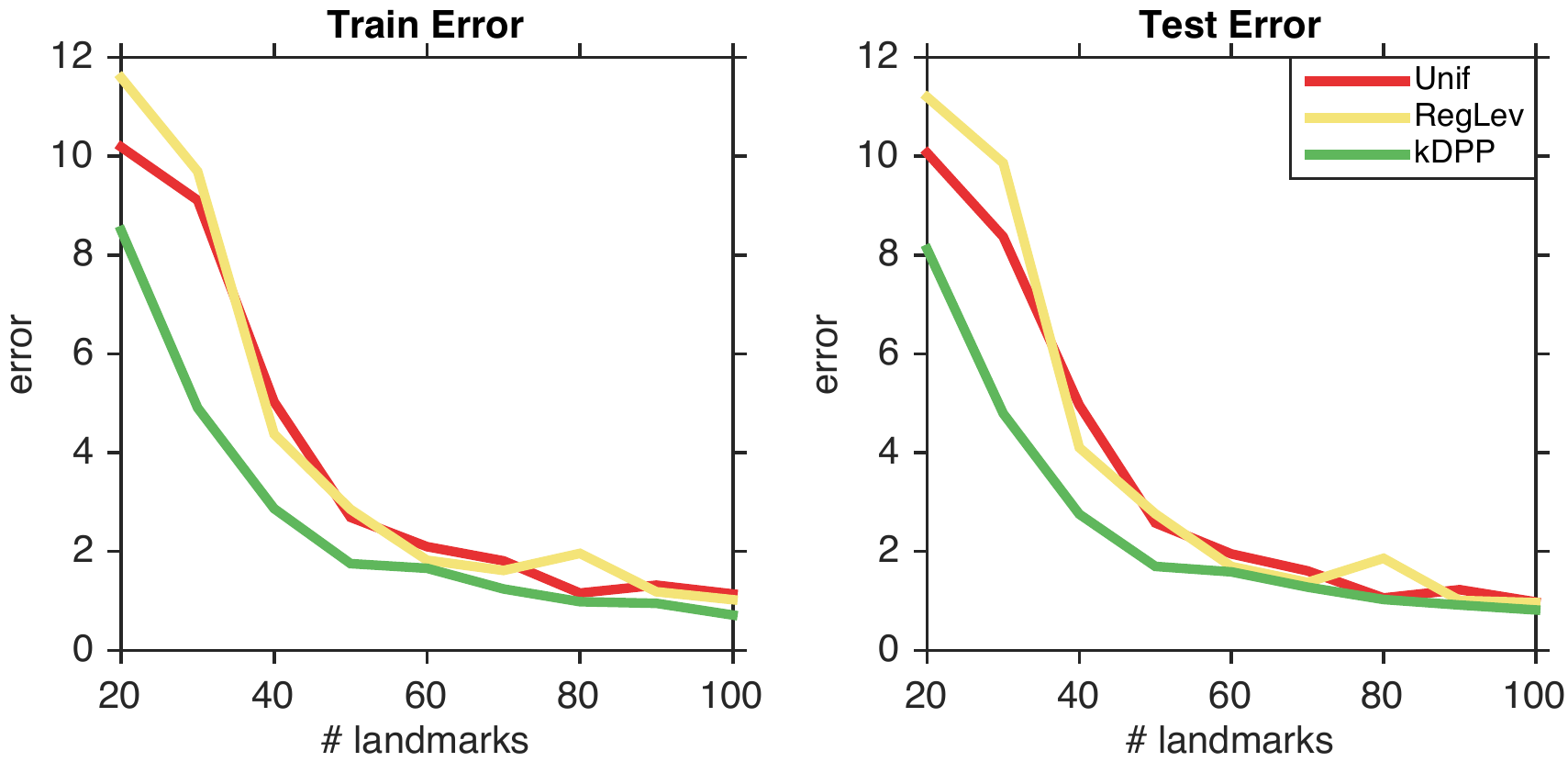}
    \caption{Bank32NH}
    \end{subfigure}%
    \begin{subfigure}{.5\textwidth}
    \centering
    \includegraphics[width=\textwidth]{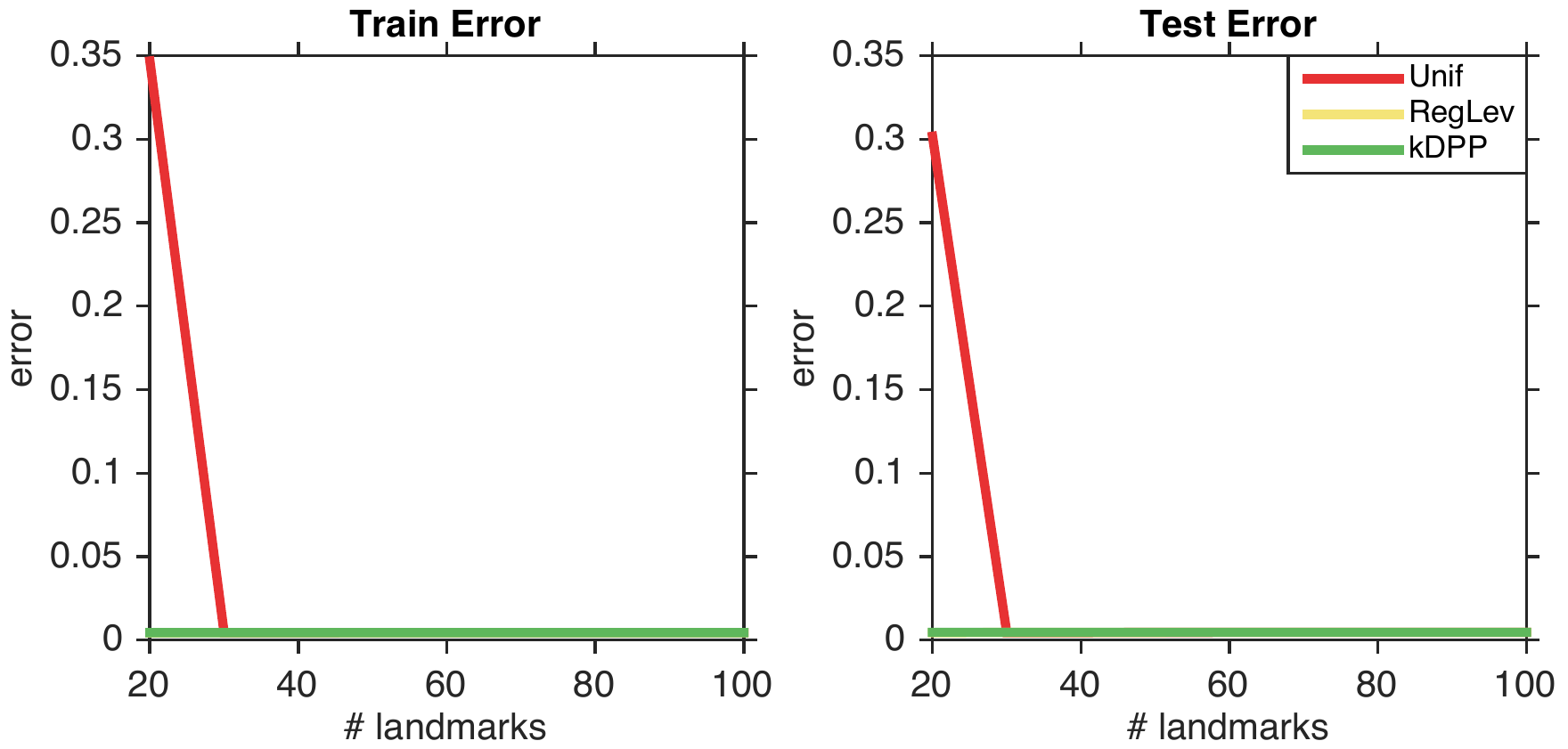}
    \caption{California Housing}
    \end{subfigure}
    
    \begin{subfigure}{.5\textwidth}
    \centering
    \includegraphics[width=\textwidth]{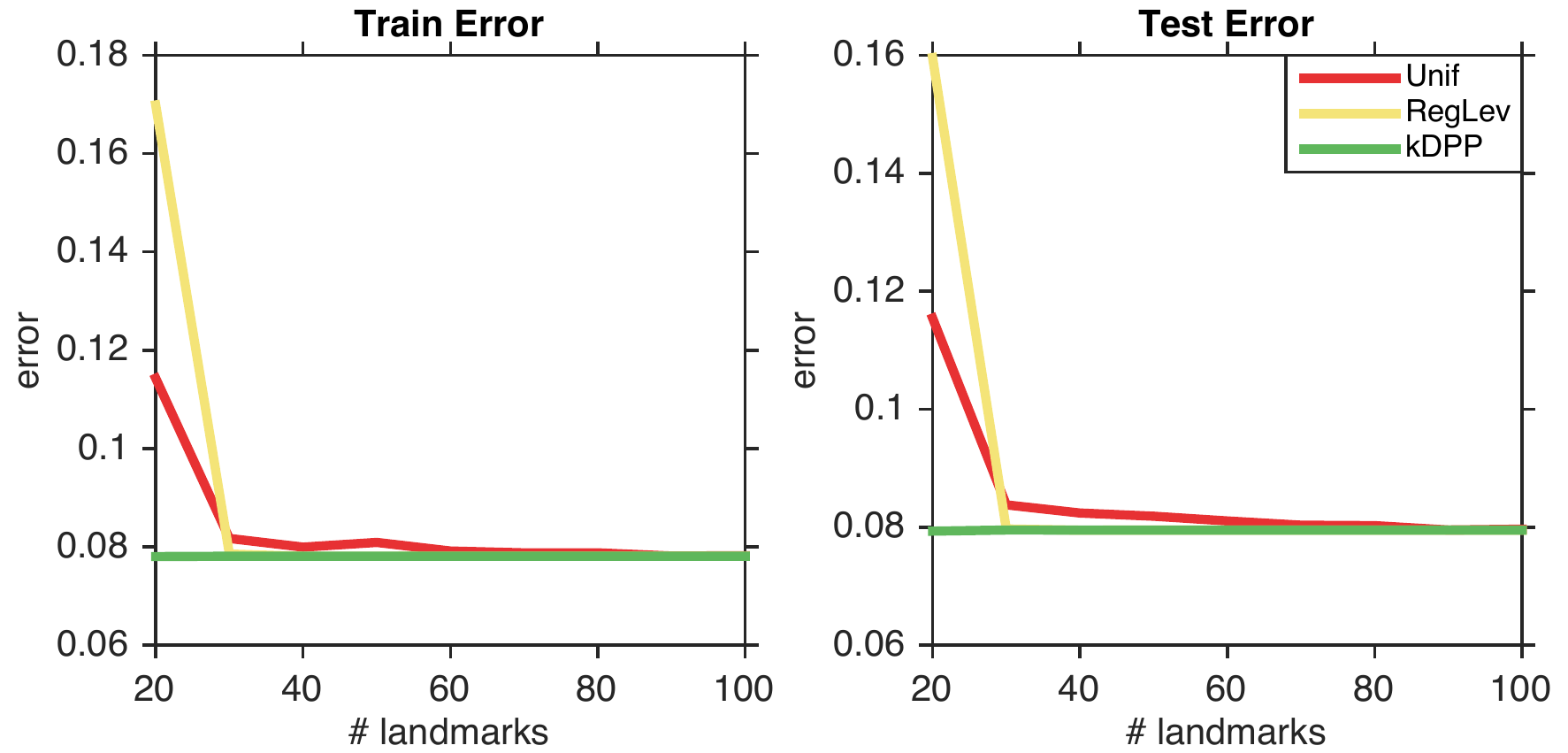}
    \caption{CompAct}
    \end{subfigure}%
    \begin{subfigure}{.5\textwidth}
    \centering
    \includegraphics[width=\textwidth]{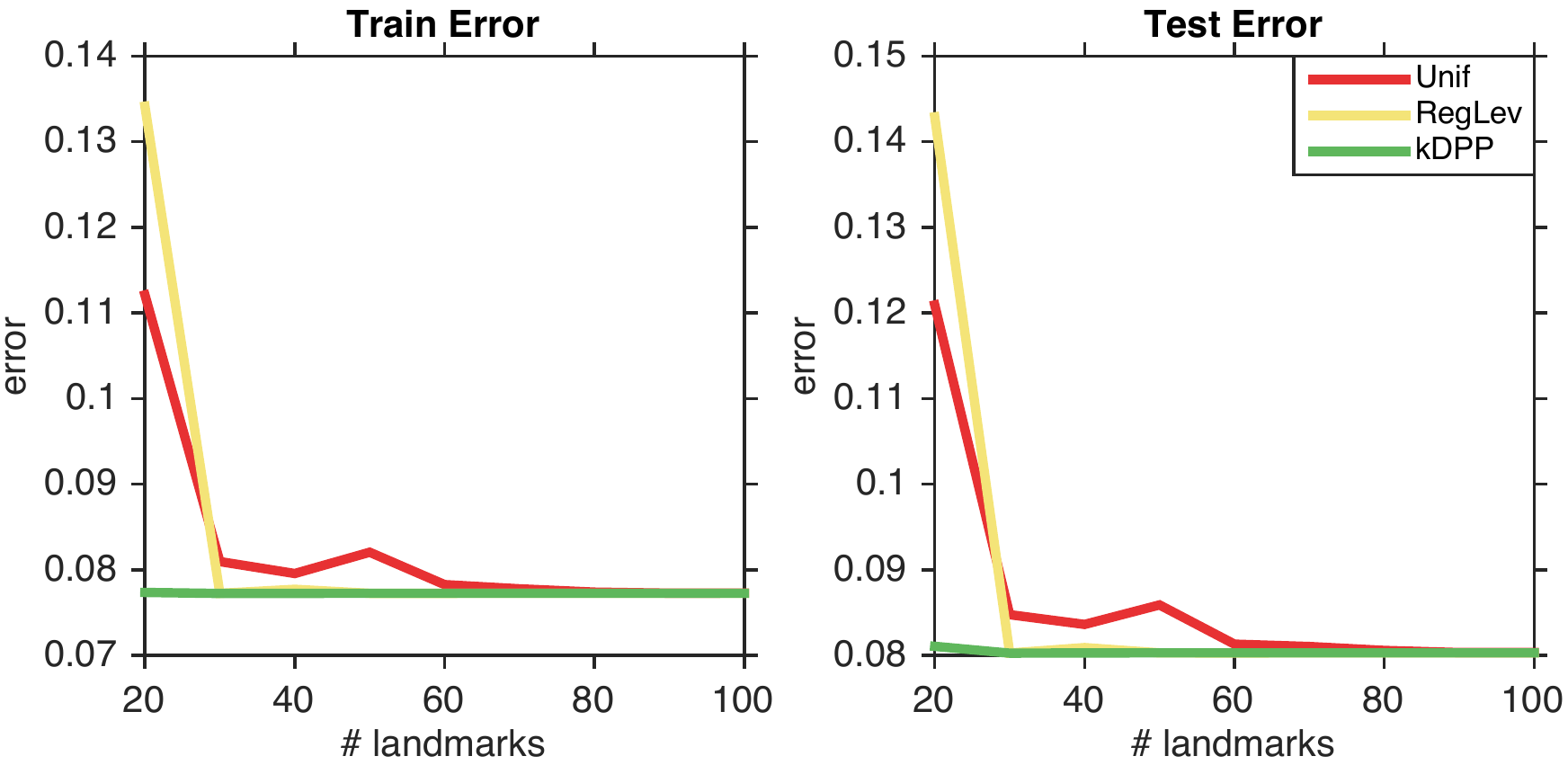}
    \caption{CompAct(s)}
    \end{subfigure}
    
    \begin{subfigure}{.5\textwidth}
    \centering	
    \includegraphics[width=\textwidth]{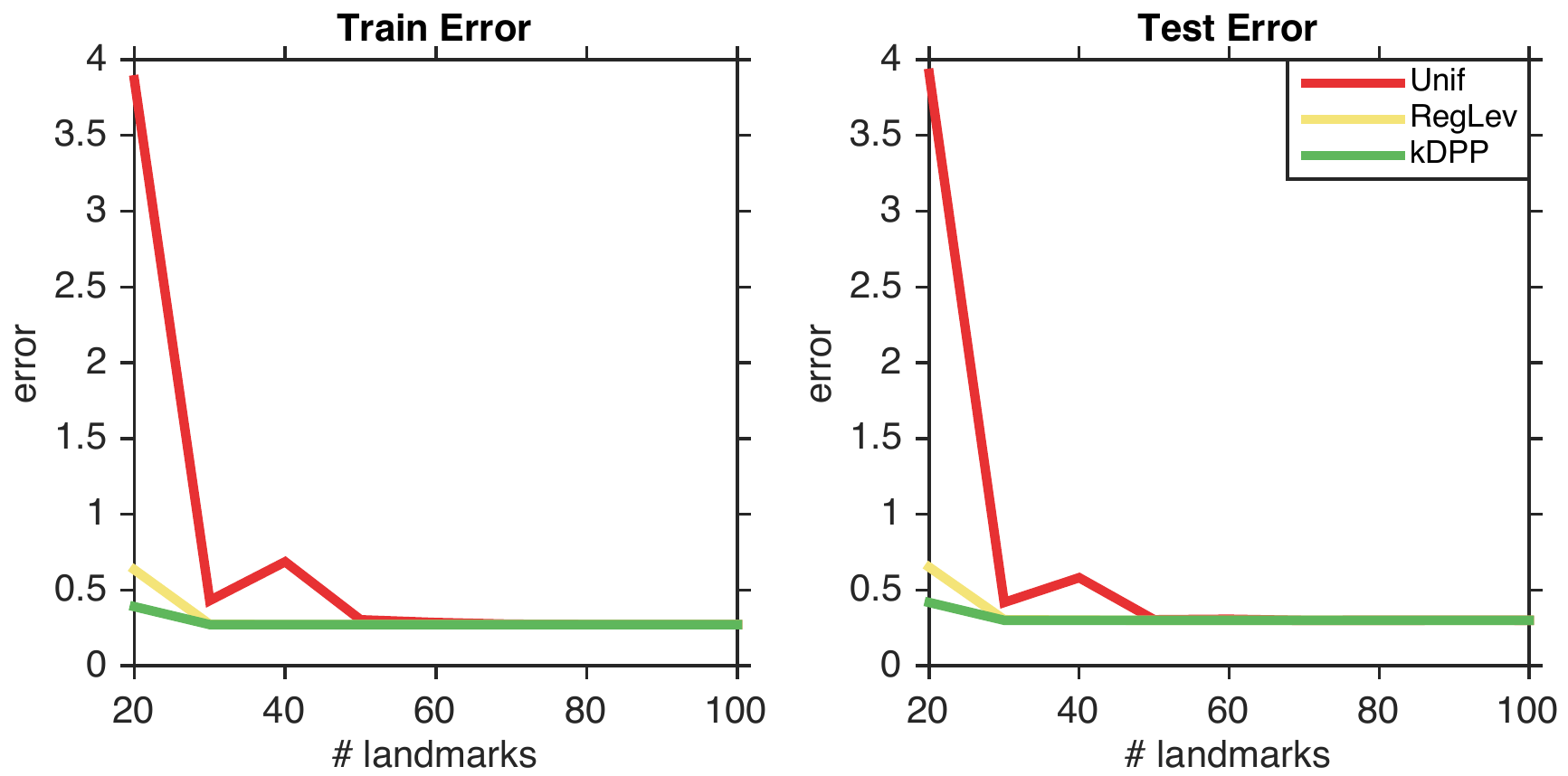}
    \caption{Elevators}
    \end{subfigure}
    
\caption{Training and test error achieved by different \nys kernel ridge regression algorithms on the remaining 7 regression datasets.}
\label{append:fig:krr_others_mc}
\end{figure}

\subsection{Mixing of Markov Chain $k$-\dpp}
\label{append:sec:conv}

We first show the mixing of the Gibbs \dpp-\nys with 50 landmarks with different performance measures: relative spectral norm error, training error and test error of kernel ridge regression in~\reffig{append:fig:conv_ailerons_50}.

We also show corresponding results with respect to 100 and 200 landmarks in~\reffig{append:fig:conv_ailerons_100} and~\reffig{append:fig:conv_ailerons_200}, so as to illustrate that for varying number of landmarks the chain is indeed fast mixing and will give reasonably good result within a small number of iterations.

\begin{figure}
\centering
	\begin{subfigure}{\textwidth}
	\centering
	\includegraphics[width=.9\textwidth]{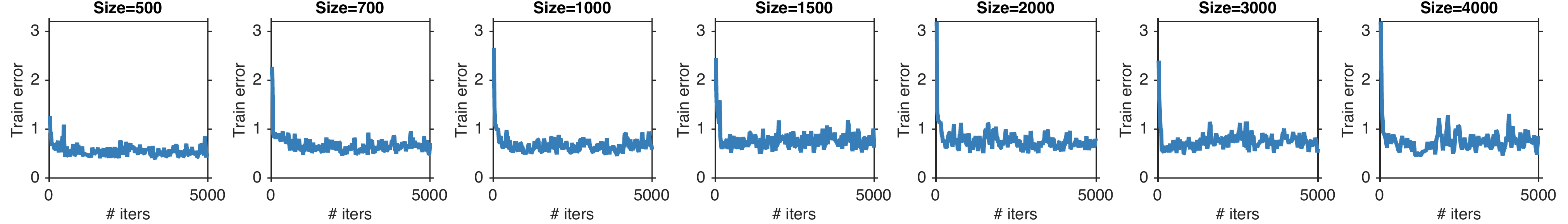}
	\caption{Training error}
	\end{subfigure}
	
	\begin{subfigure}{\textwidth}
	\centering
	\includegraphics[width=.9\textwidth]{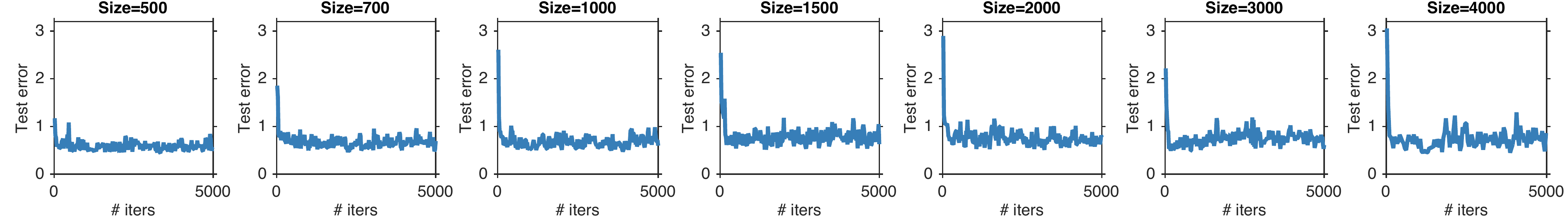}
	\caption{Test error}
	\end{subfigure}
	
	\begin{subfigure}{\textwidth}
	\centering
	\includegraphics[width=.9\textwidth]{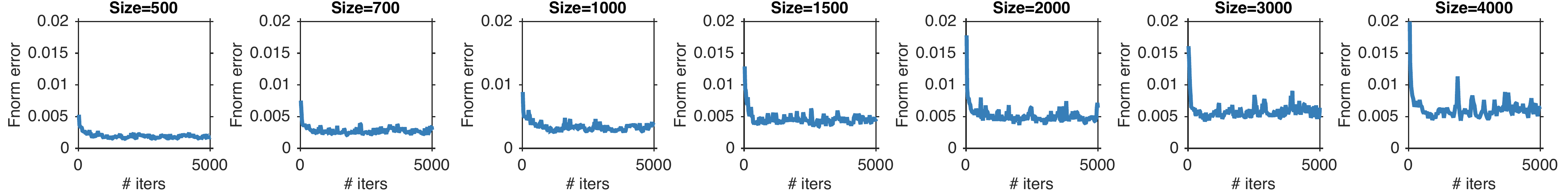}
	\caption{Relative Frobenius norm error}
	\end{subfigure}
	
	\begin{subfigure}{\textwidth}
	\centering
	\includegraphics[width=.9\textwidth]{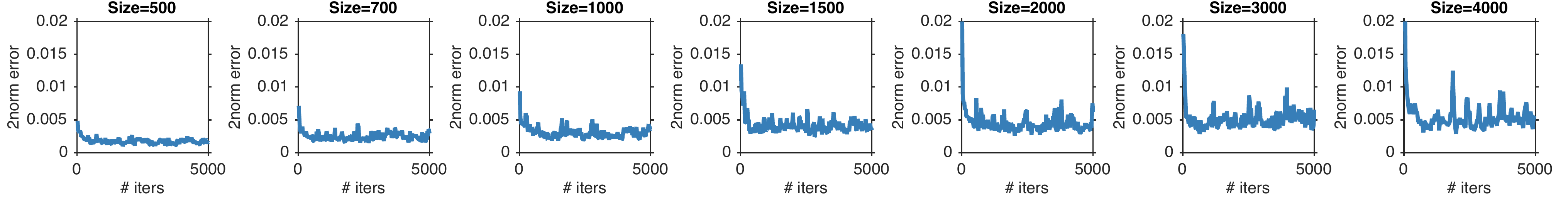}
	\caption{Relative Spectral norm error}
	\end{subfigure}
\caption{Performance of Markov chain \dpp-\nys with 50 landmarks on Ailerons. Runs for 5,000 iterations.}
\label{append:fig:conv_ailerons_50}
\end{figure}

\begin{figure}
\centering
	\begin{subfigure}{\textwidth}
	\centering
	\includegraphics[width=.9\textwidth]{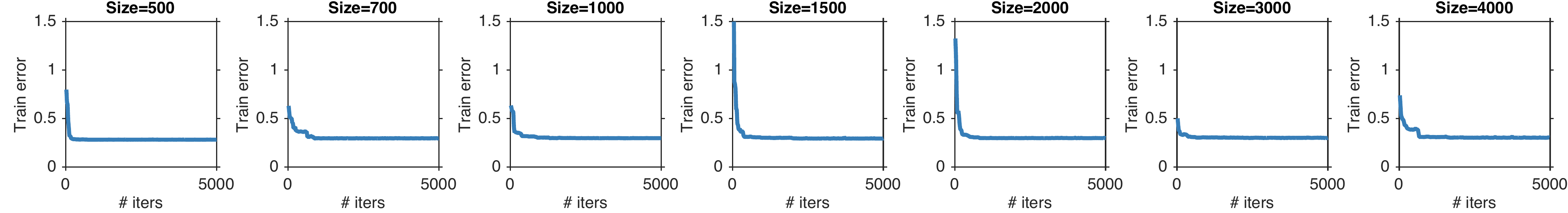}
	\caption{Training error}
	\end{subfigure}
	
	\begin{subfigure}{\textwidth}
	\centering
	\includegraphics[width=.9\textwidth]{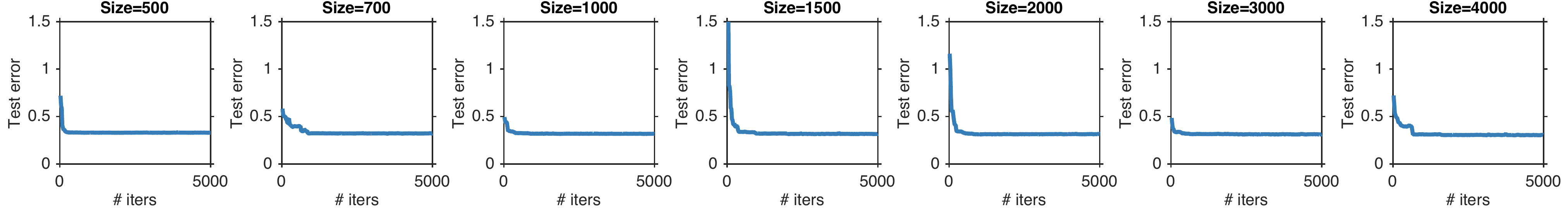}
	\caption{Test error}
	\end{subfigure}
	
	\begin{subfigure}{\textwidth}
	\centering
	\includegraphics[width=.9\textwidth]{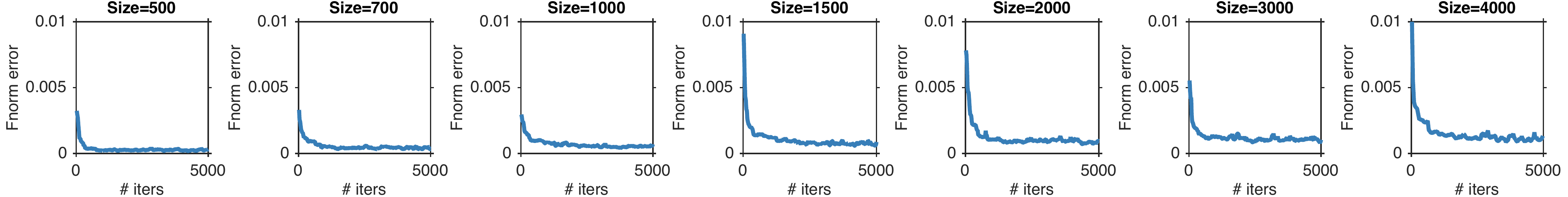}
	\caption{Relative Frobenius norm error}
	\end{subfigure}
	
	\begin{subfigure}{\textwidth}
	\centering
	\includegraphics[width=.9\textwidth]{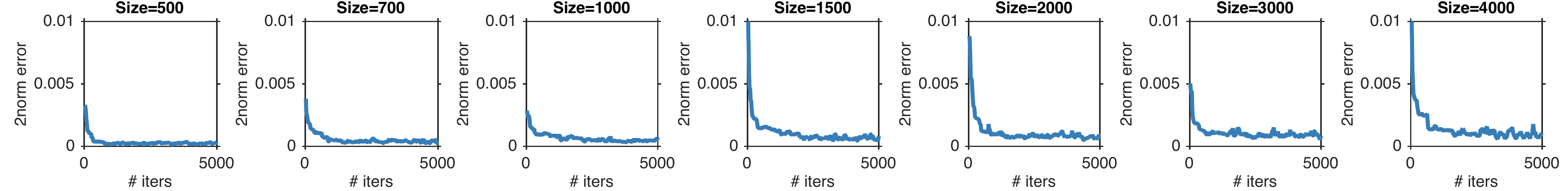}
	\caption{Relative Spectral norm error}
	\end{subfigure}
	
\caption{Performance of Markov chain \dpp-\nys with 100 landmarks on Ailerons. Runs for 5,000 iterations.}
\label{append:fig:conv_ailerons_100}
\end{figure}

\begin{figure}
\centering
	\begin{subfigure}{\textwidth}
	\centering
	\includegraphics[width=\textwidth]{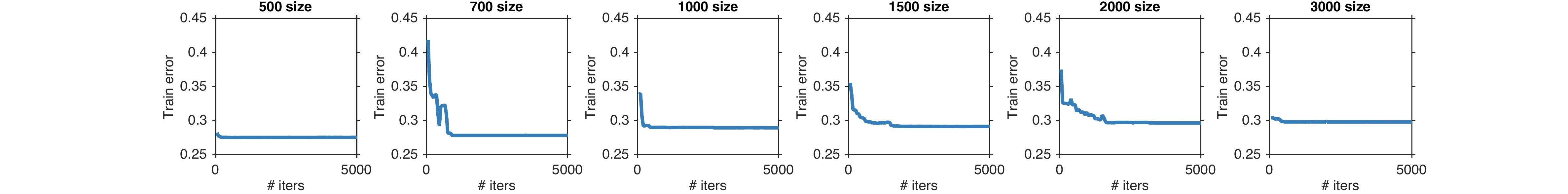}
	\caption{Training error}
	\end{subfigure}
	
	\begin{subfigure}{\textwidth}
	\centering
	\includegraphics[width=\textwidth]{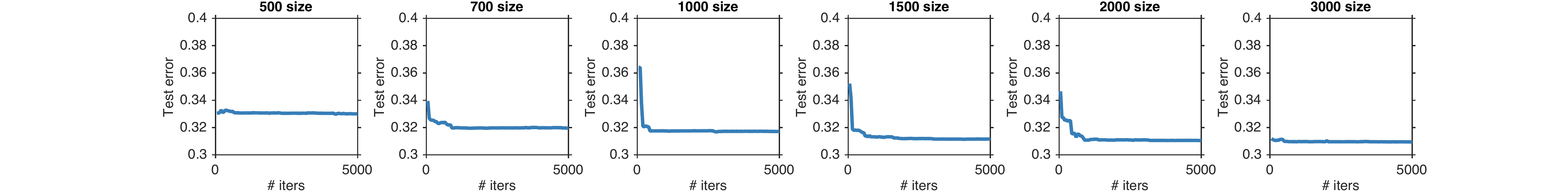}
	\caption{Test error}
	\end{subfigure}
	
	\begin{subfigure}{\textwidth}
	\centering
	\includegraphics[width=\textwidth]{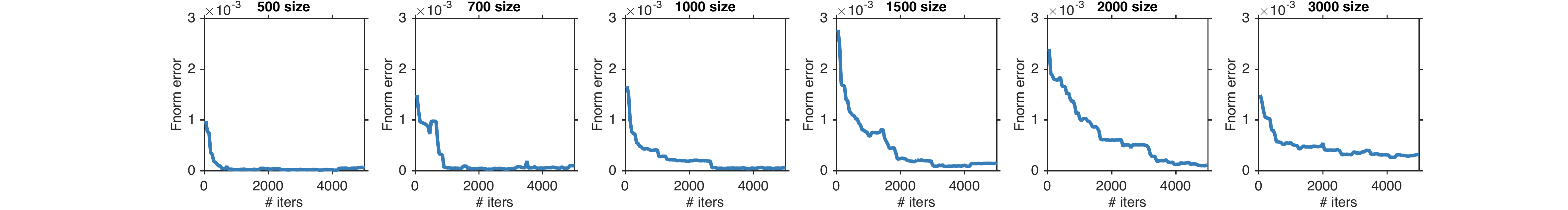}
	\caption{Relative Frobenius norm error}
	\end{subfigure}
	
	\begin{subfigure}{\textwidth}
	\centering
	\includegraphics[width=\textwidth]{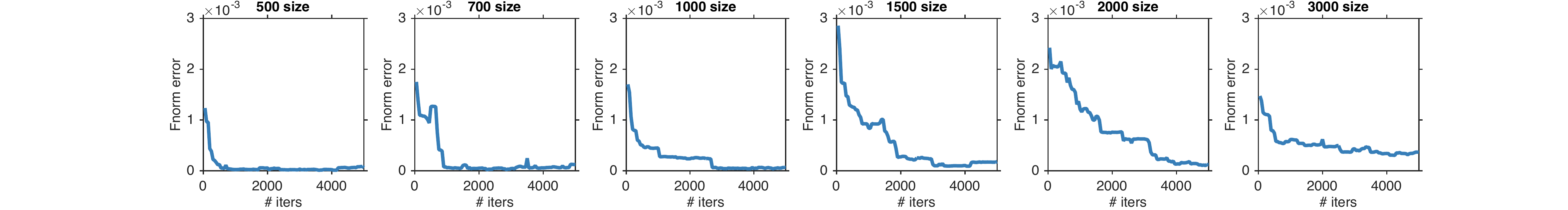}
	\caption{Relative Spectral norm error}
	\end{subfigure}
	
\caption{Performance of Markov chain \dpp-\nys with 200 landmarks on Ailerons. Runs for 5,000 iterations.}
\label{append:fig:conv_ailerons_200}
\end{figure}

\subsection{Running Time Analysis}
\label{append:sec:tradeoff}

We next show time-error trade-offs for various sampling methods on small and larger datasets with respect to Fnorm and 2norm errors. We sample 20 landmarks from Ailerons dataset of size 4,000 and California Housing of size 12,000. The result is shown in Figure~\ref{append:fig:ailerons_tradeoff_large} and Figure~\ref{append:fig:calhousing_tradeoff_large} and similar trends as the example results in the main text could be spotted: on small scale dataset (size 4,000) \ours get very good time-error trade-off. It is more efficient than \kmeans, though the error is a bit larger. While on larger dataset (size 12,000) the efficiency is further enhanced while the error is even lower than \kmeans. It also have lower variances in both cases compared to \applev and \apprlev. Overall, on larger dataset we obtain the best time-error trade-off with \ours.

\begin{figure}
\centering
	\begin{subfigure}{.5\textwidth}
	\centering
	\includegraphics[width=.9\textwidth]{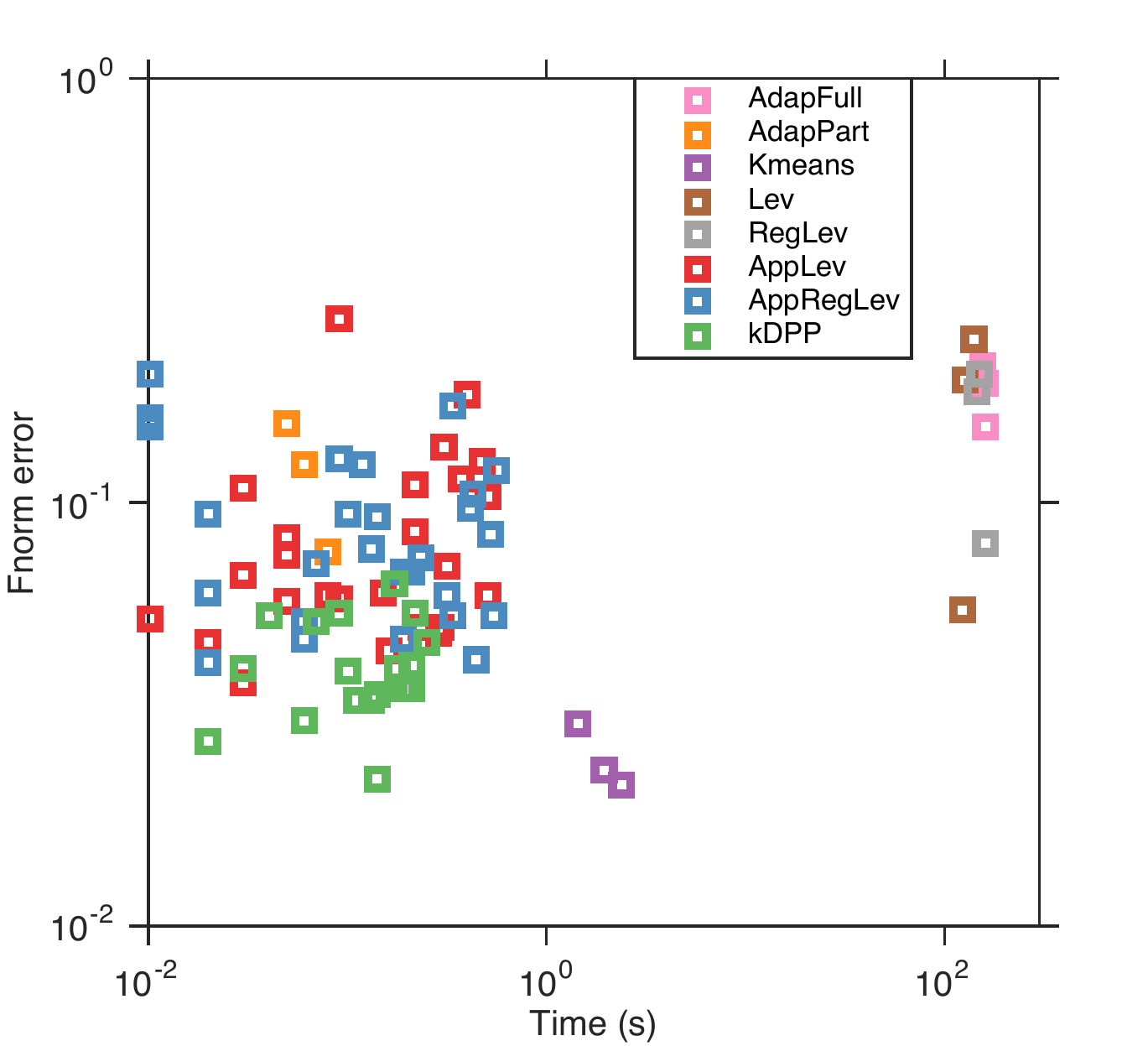}
	\caption{Fnorm Error vs. Time}
	\end{subfigure}%
	\begin{subfigure}{.5\textwidth}
	\centering
	\includegraphics[width=.9\textwidth]{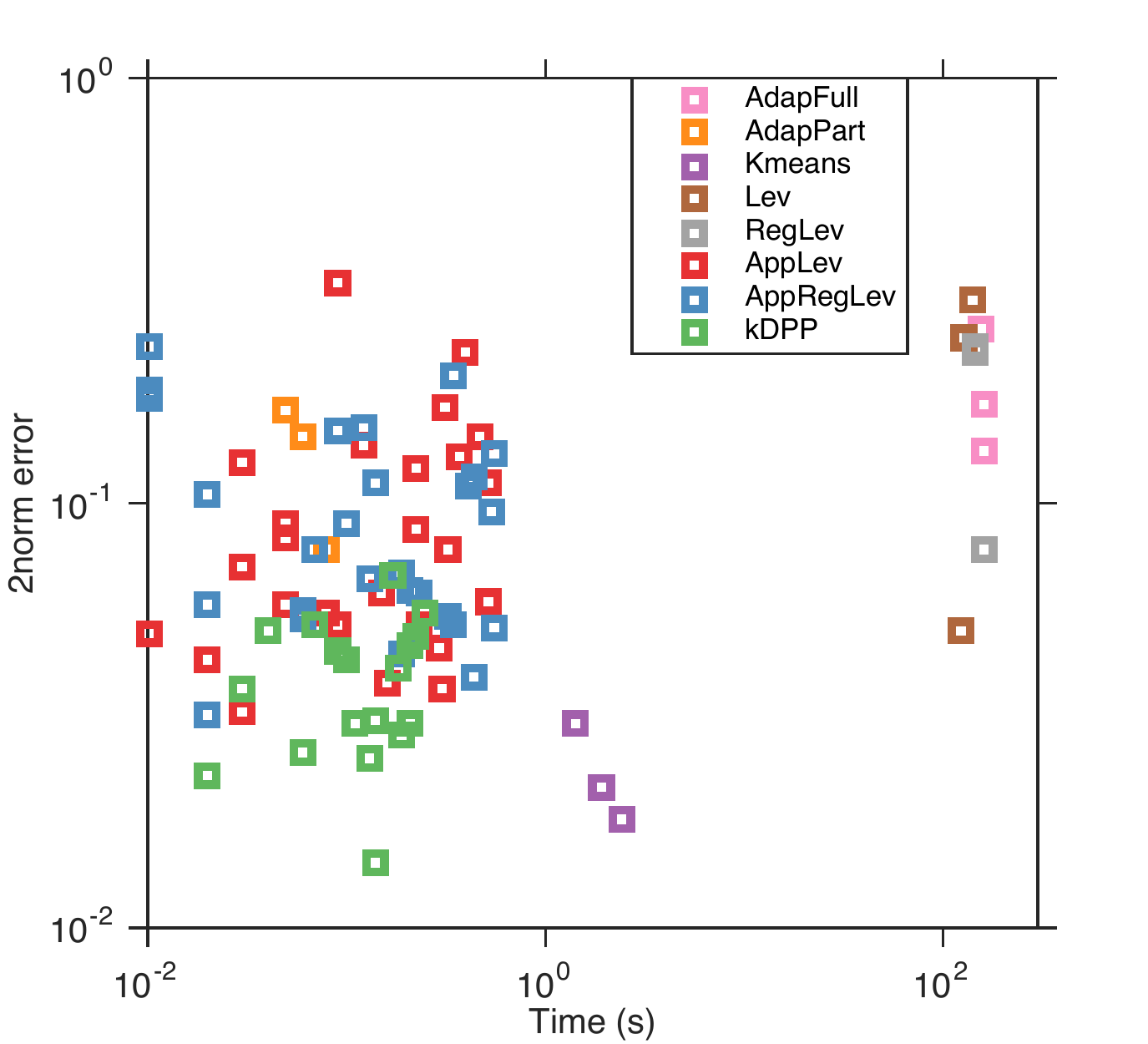}
	\caption{2norm Error vs. Time}
	\end{subfigure}
	\caption{Time-Error tradeoff with 20 landmarks on Ailerons of size 4,000. Time and Errors shown in log-scale.}
	\label{append:fig:ailerons_tradeoff_large}
\end{figure}

\begin{figure}
\centering
	\begin{subfigure}{.5\textwidth}
	\centering
	\includegraphics[width=.9\textwidth]{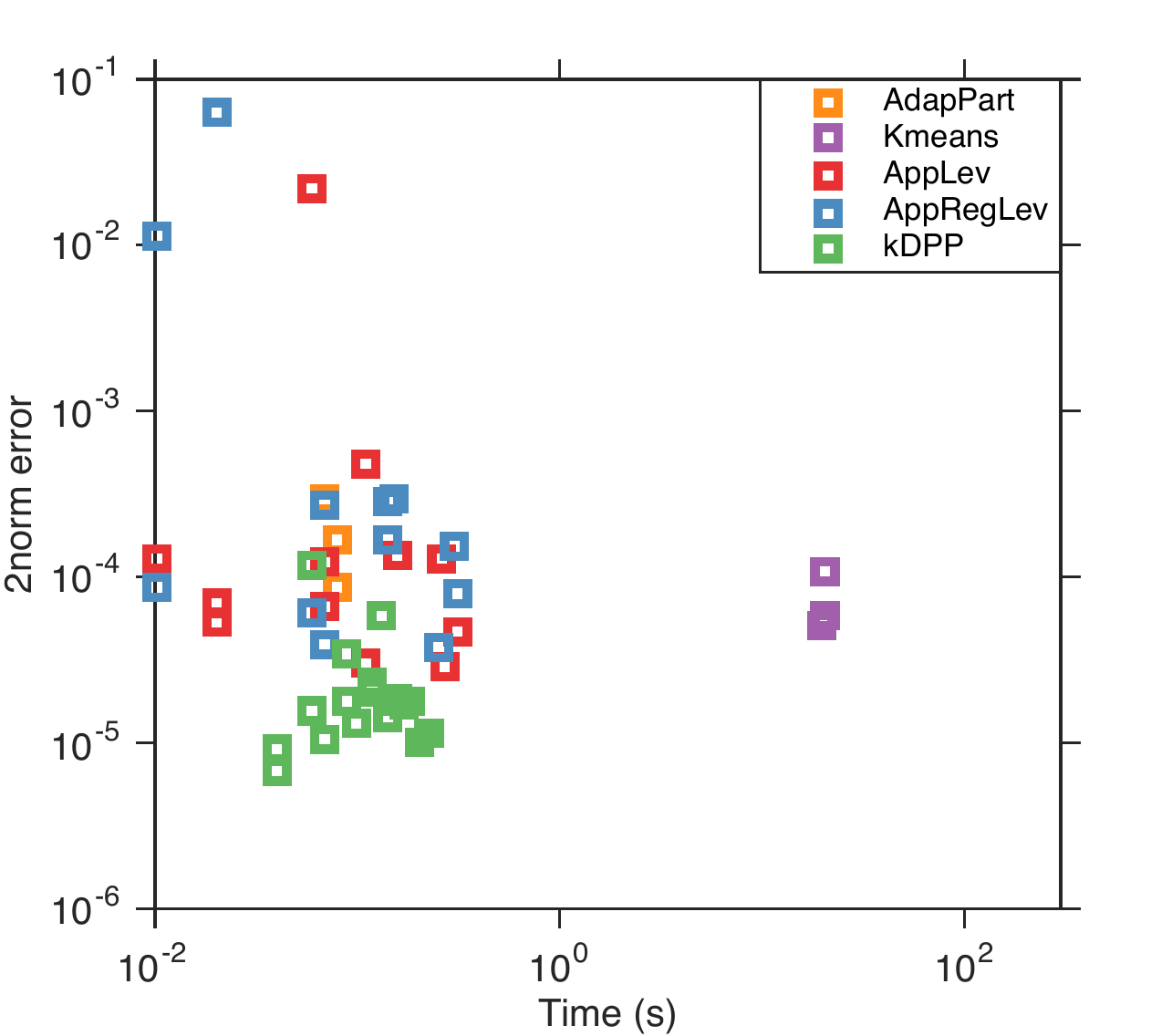}
	\caption{2norm Error vs. Time}
	\end{subfigure}%
	\begin{subfigure}{.5\textwidth}
	\centering
	\includegraphics[width=.9\textwidth]{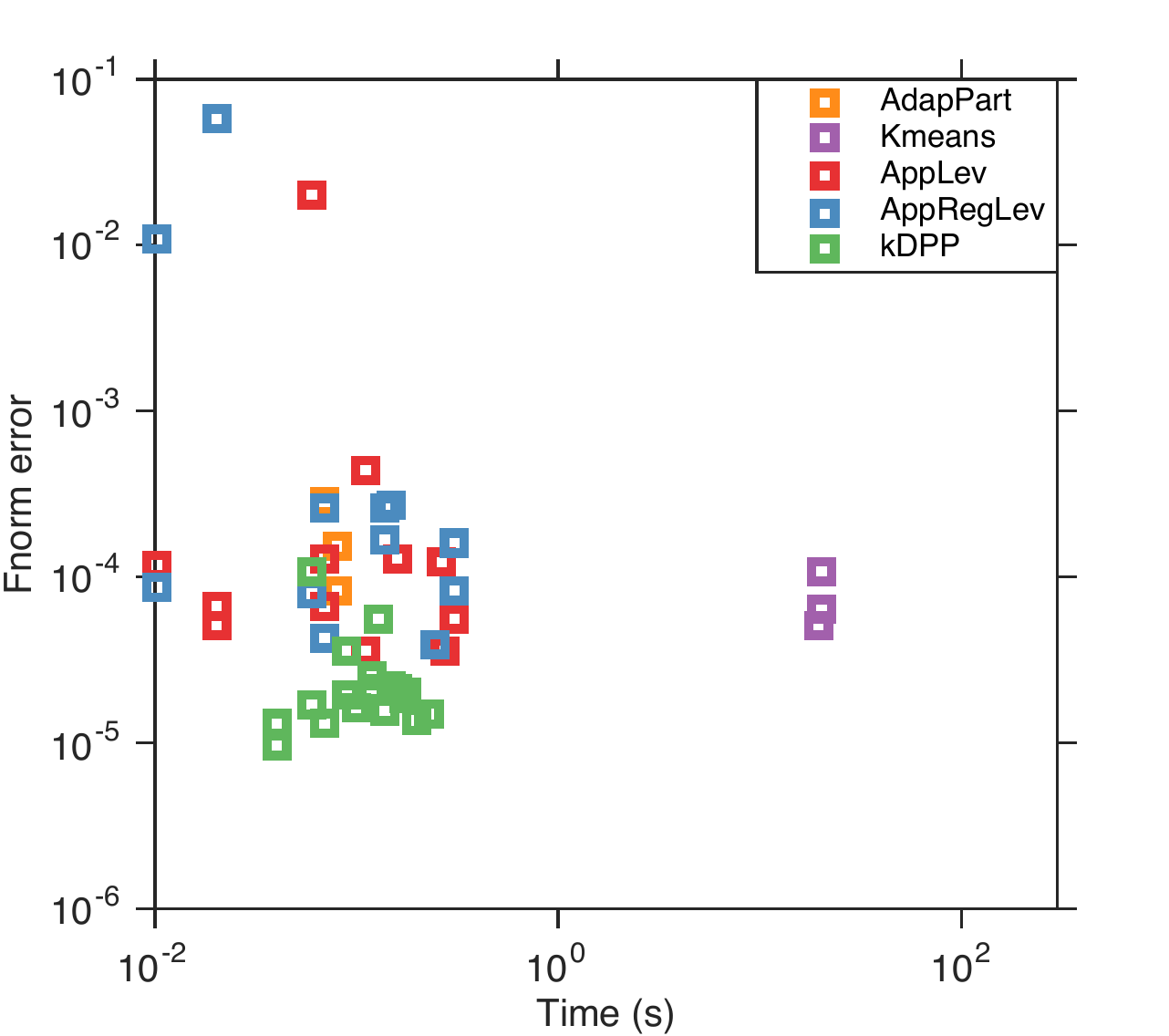}
	\caption{Training Error vs. Time}
	\end{subfigure}
	\caption{Time-Error tradeoff with 20 landmarks on California Housing of size 12,000. Time and Errors shown in log-scale. We didn't include \adap, \lev and \rlev due to their inefficiency on larger datasets.}
	\label{append:fig:calhousing_tradeoff_large}
\end{figure}

\end{appendix}

\end{document}